\documentclass[journal]{IEEEtran}

\usepackage[numbers,square,comma,sort&compress]{natbib}
\usepackage{etoolbox}
\usepackage{amsmath,amssymb,amsfonts}
\usepackage{amsthm}
\usepackage{graphicx}
\usepackage{booktabs}
\usepackage{multirow}
\usepackage{colortbl}
\usepackage{indentfirst}
\usepackage{amsmath}
\usepackage{bbding}
\usepackage{booktabs}
\usepackage{array}
\usepackage{algorithm}
\usepackage{algpseudocode}
\usepackage[table]{xcolor}
\usepackage{url}
\usepackage{xr-hyper}          
\usepackage{hyperref}

\newtheorem{proposition}{Proposition}

\makeatletter
\patchcmd{\@makecaption}{\centering\@IEEEtabletopskipstrut{\normalfont\footnotesize #1}\\{\normalfont\footnotesize\scshape #2}}{\@IEEEtabletopskipstrut{\normalfont\footnotesize #1.}~{\normalfont\footnotesize #2}}{}{\PackageWarning{main}{table caption patch failed}}
\makeatother

\newcommand{\cmark}{\textcolor{green!55!black}{\checkmark}}
\newcommand{\xmark}{\textcolor{red!70!black}{$\times$}}

\begin{document}

\title{Answer with Evidence: Consistency-Aware Grounded Visual Question Answering for Roadside Traffic Scenes}

\author{Runwei~Guan,
        Rongsheng~Hu,
        Shangshu~Chen,
        Ningwei~Ouyang,
        Shaofeng~Liang,
        Heyi~Lin,
        Jinjing~Zhu,
        Yang~Shi,
        Dongming Wu,
        Jianfeng Dong,
        Daizong~Liu,
        Henghui~Ding,
        and~Hui~Xiong,~\IEEEmembership{Fellow,~IEEE}%
\thanks{Runwei Guan, Shaofeng Liang, Heyi Lin, and Hui Xiong are with The Hong Kong University of Science and Technology (Guangzhou), Guangzhou, China. Hui Xiong is also with The Hong Kong University of Science and Technology, Hong Kong SAR, China. Corresponding author: xionghui@hkust-gz.edu.cn.}%
\thanks{Rongsheng Hu is with Jiangnan University, Wuxi, China.}%
\thanks{Shangshu Chen is with Central China Normal University, Wuhan, China.}%
\thanks{Ningwei Ouyang is with University of Liverpool, Liverpool, UK.}%
\thanks{Jinjing Zhu and Dongming Wu are with CUHK, Hong Kong SAR, China.}%
\thanks{Yang Shi is with Peking University, Beijing, China.}%
\thanks{Jianfeng Dong is with Zhejiang Gongshang University, Hangzhou, China. }
\thanks{Daizong Liu is with Wuhan University, Wuhan, China.}%
\thanks{Henghui Ding is with Fudan University, Shanghai, China.}%
}

\markboth{}
{Guan \MakeLowercase{\textit{et al.}}: Answer with Evidence: Consistency-Aware Grounded VQA for Roadside Traffic Scenes}

\maketitle

\begin{abstract}
Roadside traffic reasoning requires every free-form textual claim to be grounded in visual evidence. Existing grounded multimodal large language models (MLLMs) frequently exhibit say-point mismatch, wherein the textual response contradicts the localized bounding boxes. Standard evaluation metrics that assess answers and bounding boxes independently leave this failure unpenalized. We attribute this mismatch to the conventional answer-then-ground factorization, which commits to a numerical claim prior to enumerating target objects. To quantify this issue, we construct RoadSceneVQA-G, a benchmark comprising 34.7K question-answer pairs where each free-form answer is paired with its supporting box set, and propose the Answer-Grounding Consistency (AGC) evaluation suite. To resolve it, we introduce Enumerate-then-Answer (EtA), which reverses the decoding order so that answer-evidence alignment is inherently enforced by the output structure, along with Enumeration-Consistent Policy Optimization (ECPO), a reinforcement learning framework that leverages the union of multiple rollouts as a recall teacher without requiring ground-truth boxes. EtA boosts say-point consistency from 26.6\% to 93.7\% and grounding F1 from 52.2\% to 73.0\%, while ECPO further elevates F1 to 75.6\% without per-box supervision. On gRefCOCO, our framework outperforms leading baselines, demonstrating strong transferability beyond traffic scenes. The project is available at \url{https://github.com/GuanRunwei/RoadSceneVQA-G}.
\end{abstract}


\section{Introduction}
\label{sec:intro}

\IEEEPARstart{I}{ntelligent} transportation systems are moving from passive perception toward interactive traffic management~\cite{guan2026roadscenevqa,zhou2025tumtraffic,zhou2024vision}. Operators query high-angle roadside cameras in natural language, and the answers inform downstream physical decisions~\cite{ye2022rope3d,yu2022dair}. Grounded multimodal large language models (MLLMs) fit this setting because they return a free-form answer together with the bounding boxes that support it~\cite{guan2026roadscenevqa,sima2024drivelm,chen2025automated}. Hallucination and weak grounding, however, still limit the reliability of vision-language models in driving scenes~\cite{liu2026rar,fu2025boosting}, whereas safety-critical traffic management requires every assertion to be backed by verifiable visual evidence.

\begin{figure}[!t]
\centering
\includegraphics[width=\columnwidth]{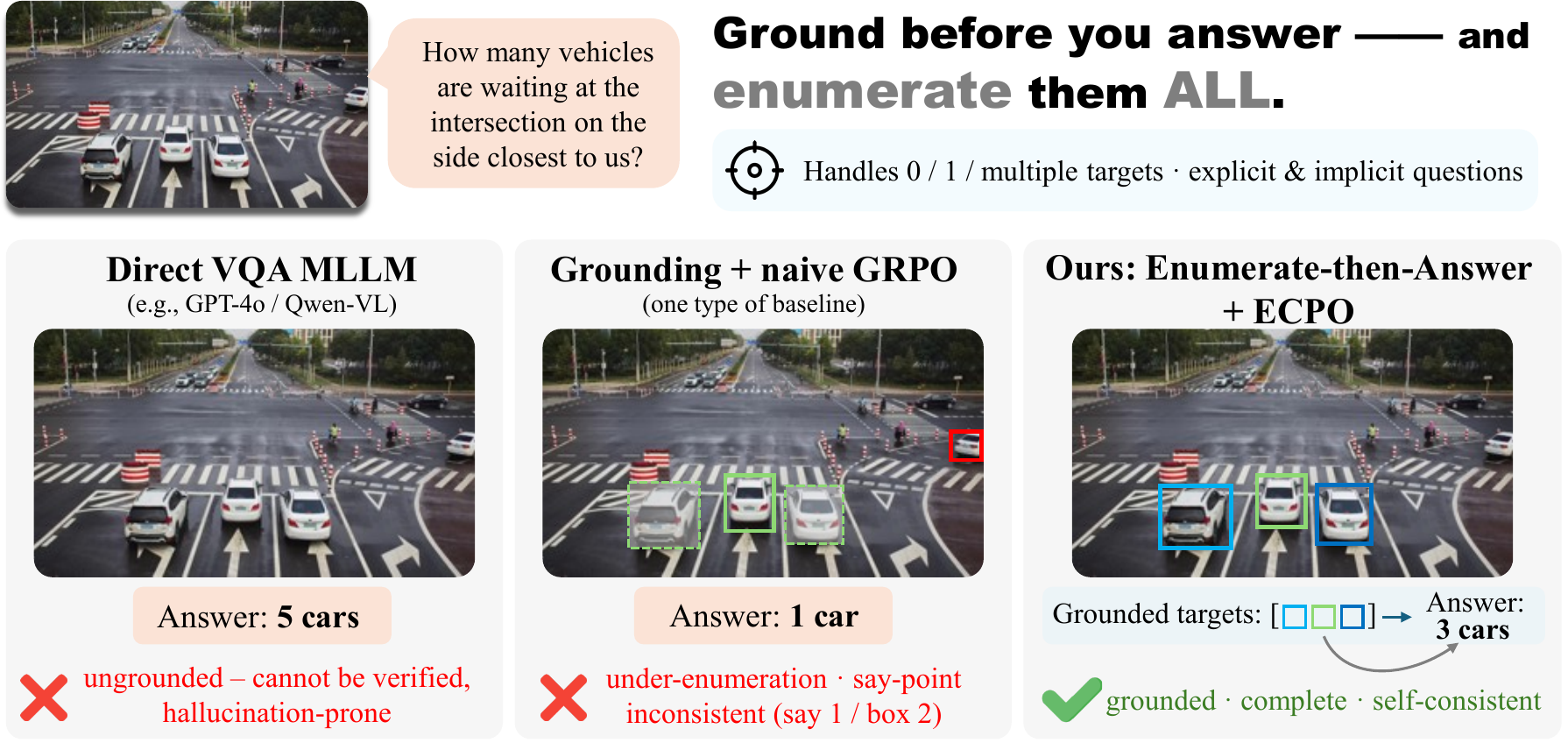}
\vspace{-5mm}
\caption{Grounded VQA in a roadside scene. A direct VQA MLLM answers without evidence. A grounding model trained with naive GRPO states one car but localizes two (say-point mismatch). EtA with ECPO first localizes all evidence, then reads the answer from it.}
\label{fig:teaser}
\end{figure}

To address this issue, we identify a key failure mode in grounded models termed \emph{say-point mismatch}, where textual answers contradict localized bounding boxes. As shown in Fig.\ref{fig:teaser}, a model claims one vehicle is waiting while localizing two. Standard metrics evaluate text~\cite{vedantam2015cider} and bounding boxes~\cite{yu2016modeling} independently~\cite{rohrbach2018object,li2023evaluating}, leaving this contradiction unpenalized. We attribute this to the answer-then-ground factorization, which samples count tokens before enumerating objects and thus commits to a numerical claim prior to producing spatial evidence. Furthermore, roadside queries require diverse grounding types that conventional paradigms fail to unify in a single inference step.

\begin{table*}[!t]
\centering
\caption{Comparison of RoadSceneVQA-G with existing benchmarks. \textbf{Open A.}: free-form textual answers. \textbf{Grd.}: grounding annotations tied to the answer. \textbf{Multi.}: multi-object boxes per QA pair. \textbf{$\emptyset$-grd.}: explicit no-ground cases. \textbf{Consist.}: answer-grounding consistency is directly measurable.}
\label{tab:dataset_comparison}
\vspace{-3mm}
\setlength{\tabcolsep}{10pt}
\begin{tabular}{llcccccccc}
\toprule
\textbf{Datasets} & \textbf{Venues} & \textbf{Domains} & \textbf{\# Pairs} & \textbf{Open A.} & \textbf{Grd.} & \textbf{Multi.} & \textbf{$\emptyset$-grd.} & \textbf{Consist.} \\
\midrule
\multicolumn{9}{l}{\textit{General VQA and visual grounding benchmarks}} \\
GQA~\cite{hudson2019gqa} & CVPR 2019 & General scenes & $\sim$22M & \xmark & \cmark & \cmark & \xmark & \xmark \\
RefCOCO+~\cite{yu2016modeling} & ECCV 2016 & General objects & $\sim$20K & \xmark & \cmark & \xmark & \xmark & \xmark \\
gRefCOCO~\cite{ding2026grex} & IJCV 2026 & General objects & $\sim$260K & \xmark & \cmark & \cmark & \cmark & \xmark \\
\midrule
\multicolumn{9}{l}{\textit{Traffic-scene VQA benchmarks}} \\
DriveBench~\cite{xie2025vlms} & ICCV 2025 & Autonomous driving & $\sim$60K & \cmark & \xmark & \xmark & \xmark & \xmark \\
TUMTraffic-VideoQA~\cite{zhou2025tumtraffic} & ICML 2025 & Roadside traffic & -- & \cmark & \xmark & \xmark & \xmark & \xmark \\
RoadSceneVQA~\cite{guan2026roadscenevqa} & AAAI 2026 & Roadside traffic & 34.7K & \cmark & \xmark & \xmark & \xmark & \xmark \\
\rowcolor{gray!30} \textbf{RoadSceneVQA-G (ours)} & -- & Roadside traffic & 34.7K & \cmark & \cmark & \cmark & \cmark & \cmark \\
\bottomrule
\end{tabular}
\end{table*}

Systematically investigating this failure requires diagnostic data that existing benchmarks fail to provide. Traffic visual question answering (VQA) datasets~\cite{qian2024nuscenes,sima2024drivelm,zhou2025tumtraffic,guan2026roadscenevqa} evaluate free-form text without spatial annotations, whereas visual grounding benchmarks restrict outputs to predefined targets without open-ended answers. We therefore build RoadSceneVQA-G on RoadSceneVQA imagery~\cite{guan2026roadscenevqa}, where every free-form answer is aligned with its witnessing box set of zero, one, or more instances.

To resolve these limitations, we address say-point mismatch at its structural source. \textbf{Enumerate-then-Answer (EtA)} inverts the decoding order and conditions the answer on an explicitly localized box sequence, so that a stated count becomes a deterministic readout of the evidence and queries without targets are answered with an empty box set. \textbf{Enumeration-Consistent Policy Optimization (ECPO)} further reduces under-enumeration in dense scenes by pooling the boxes of multiple sampled rollouts into a consensus set that serves as a label-free recall teacher.

Our main contributions are structured as follows:
\begin{itemize}
\item \textbf{A consistency-aware grounded VQA benchmark:} We formalize say-point mismatch and present RoadSceneVQA-G, 34.7K open-ended roadside questions whose answers are paired with their witnessing box sets, with the Answer-Grounding Consistency suite that measures answer-evidence agreement.
\item \textbf{A consistent output structure:} EtA makes answer-evidence consistency a structural property of the output serialization, which we characterize theoretically against the answer-then-ground factorization.
\item \textbf{Label-free recall optimization:} ECPO recovers omitted objects without per-box annotations. Experiments on our proposed RoadSceneVQA-G, gRefCOCO \cite{ding2026grex}, and CARPK \cite{hsieh2017drone} show consistent gains in consistency, grounding quality, and joint correctness.
\end{itemize}

In the remainder, Section~\ref{sec:related} reviews related work, Section~\ref{sec:dataset} presents RoadSceneVQA-G and the AGC suite, Section~\ref{sec:method} details EtA and ECPO, Section~\ref{sec:experiments} reports experiments, and Section~\ref{sec:conclusion} concludes the paper.

\section{Related Work}
\label{sec:related}

\subsection{Grounded Multimodal Large Language Models}
Referring expression comprehension localizes image regions described by natural language~\cite{suo2022rethinking}. Grounded MLLMs extend open-ended VQA by emitting boxes as coordinate tokens or structured text, as established by Shikra~\cite{chen2023shikra}, Kosmos-2~\cite{peng2024grounding}, Ferret~\cite{you2024ferret}, and the Qwen-VL family~\cite{bai2025qwen2,bai2025qwen3,bai2023qwenvlversatilevisionlanguagemodel}. Grounding is usually applied to every query, where boxes form an output channel largely independent of the answer. Adaptive variants rely on keyword rules or an auxiliary router~\cite{wu2025routing}. Another line couples grounding with reasoning. GRIT~\cite{fan2026grit} interleaves boxes within the reasoning chain, Woodpecker~\cite{yin2024woodpecker} verifies responses post hoc with an external detector, V$^{*}$\cite{wu2024v} iteratively crops and re-queries the image, CogCoM\cite{cogcom2024} grounds, counts, and crops before answering, and VisCoT~\cite{shao2024visual} annotates visual evidence behind each reasoning step. Explanatory VQA further justifies answers with multimodal explanations~\cite{xie2024knowledge}. In these methods, boxes serve as operands of an intermediate trace rather than witnesses of the final answer. None of them requires the stated count to match localized instances, leaving answer-grounding consistency unmeasured. Counting methods~\cite{ranjan2021learning,jiang2023clip,du2023redesigning} estimate counts from density maps or point annotations but neither answer open questions nor expose per-instance boxes. EtA instead derives the answer from the enumerated box set.

\begin{figure*}[!t]
\centering
\includegraphics[width=\textwidth]{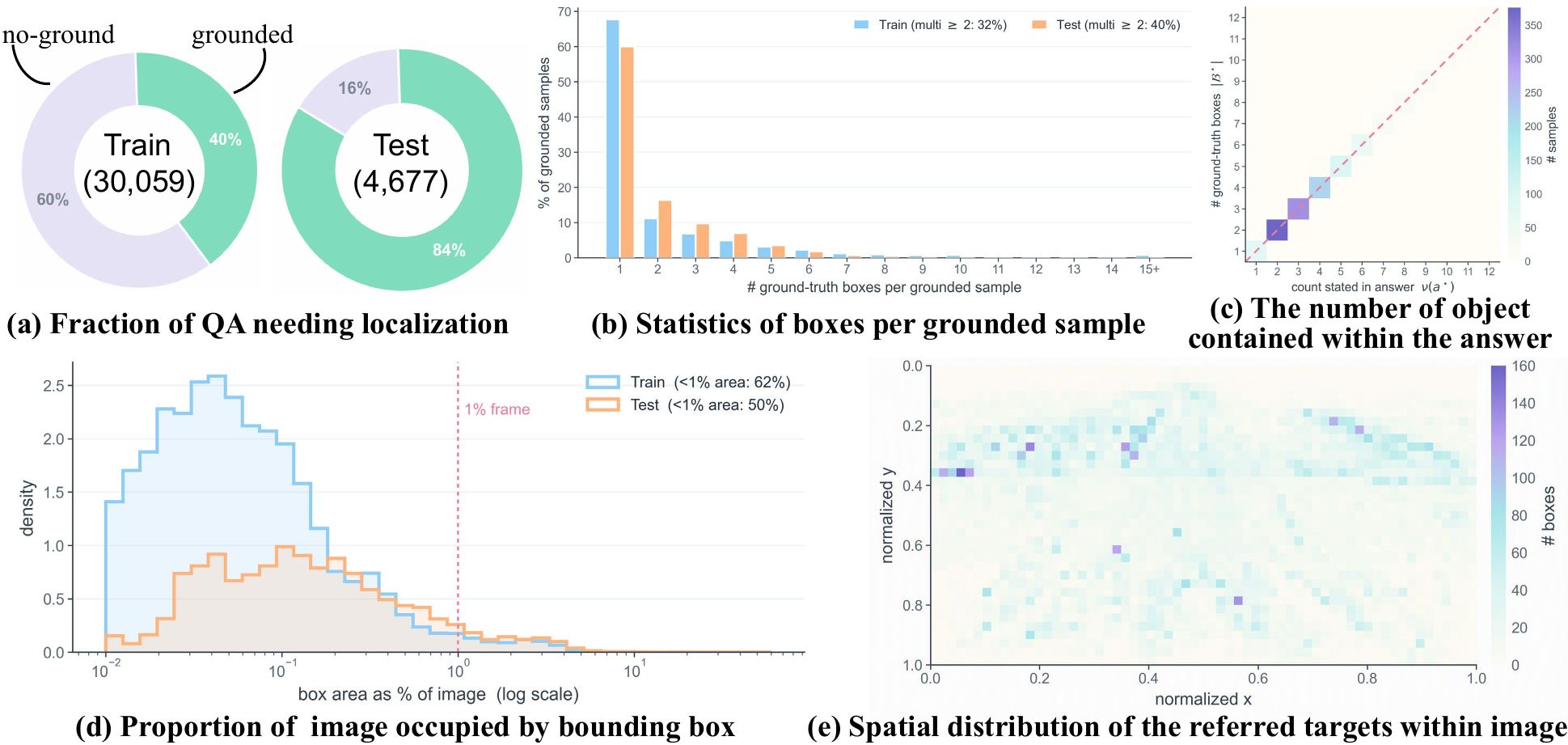}
\vspace{-5mm}
\caption{Statistics of RoadSceneVQA-G. (a) Share of questions that require localization in each split. (b) Distribution of the number of ground-truth boxes per grounded question. (c) Count stated in the answer versus the number of ground-truth evidence boxes. (d) Box area as a percentage of the image area (log scale). (e) Spatial distribution of the referred targets over normalized image coordinates.}
\label{fig:dataset_1}
\end{figure*}

\subsection{Reinforcement Learning for Visual Grounding}
Following group-relative policy optimization (GRPO) in language reasoning~\cite{shao2024deepseekmath,guo2025deepseek}, recent work applies rule-based rewards to visual grounding. UniVG-R1~\cite{bai2025univg} uses an mIoU reward, VLM-R1~\cite{shen2025vlm} designs task-specific rewards, DIP-R1~\cite{park2025dip} adds step-wise counting and IoU rewards, SpatialThinker~\cite{batra2025spatialthinker} combines count, grounding, and answer rewards, and Ground-R1 and SRPO~\cite{cao2025ground,wan2026srpo} re-weight credit by region scale. Self-consistency decoding~\cite{wangself} and consensus-filtered or self-rewarded policy optimization~\cite{chen2026grpo,zuo2026ttrl} aggregate multiple samples, typically by majority vote over a discrete answer. A vote selects one answer, whereas recovering under-enumerated objects requires completing a set. ECPO therefore takes the union of the box sets of multiple rollouts and uses the well-supported clusters as a recall teacher without ground truth.

\subsection{Visual Question Answering in Traffic Scenes}
Language-driven understanding of driving scenes has been studied in several forms. Referring benchmarks such as Talk2Car~\cite{deruyttere2019talk2car} and Talk2Radar~\cite{guan2024talk2radar} localize the object described by a command or an expression, and NuPrompt~\cite{wu2025language} extends language prompts to multiple objects tracked across frames. These tasks output boxes but no free-form answer. Traffic VQA has progressed from vehicle-centric benchmarks, including NuScenes-QA~\cite{qian2024nuscenes}, DriveLM~\cite{sima2024drivelm}, and CODA-LM~\cite{chen2025automated}, to roadside settings such as TUMTraffic-VideoQA~\cite{zhou2025tumtraffic} and RoadSceneVQA~\cite{guan2026roadscenevqa}. NuScenes-QA builds questions from scene graphs, DriveLM links perception, prediction, and planning questions in a graph, CODA-LM targets corner cases, and MPDrive~\cite{zhang2025mpdrive} expresses coordinates through visual markers to improve spatial understanding. As summarized in Table~\ref{tab:dataset_comparison}, these VQA benchmarks are evaluated mainly with text-generation metrics, and grounding annotations, when present, are not tied to the answer. DriveBench~\cite{xie2025vlms} further shows that vision-language models can produce plausible driving answers from textual cues rather than visual evidence, a failure that text metrics cannot detect. We therefore formulate roadside VQA as grounded enumeration and add the AGC consistency axis to standard answer-quality metrics, so that an answer is credited only when its evidence supports it.

\section{RoadSceneVQA-G Dataset}
\label{sec:dataset}

RoadSceneVQA-G builds on the roadside imagery and object annotations of Rope3D~\cite{ye2022rope3d} alongside the question coverage of RoadSceneVQA~\cite{guan2026roadscenevqa}. Each sample pairs a traffic image and an open-ended question with a free-form answer and its witnessing box set. Because this set can contain multiple boxes, a single box, or none, the grounding decision, enumeration completeness, and answer-evidence alignment all become directly observable (Table~\ref{tab:dataset_comparison}).

\subsection{Dataset Construction}
\label{sec:construction}

Tight object boxes already exist in the source data, so the annotation focuses on authoring questions and binding each answer to its supporting boxes (Fig.~\ref{A-fig:annotation_tools} in Appendix~\ref{A-app:dataset}). \textit{1) Candidate QA generation:} GPT-5~\cite{openai2025gpt5} drafts question-answer pairs on existence, counting, identification, and selection for each frame. These drafts are never used as ground truth. \textit{2) Evidence loading:} The existing boxes and category labels are overlaid on the image as a selectable candidate pool. \textit{3) Grounded annotation:} The annotator corrects each question and answer and selects the subset of boxes, if any, that constitutes the evidence. Questions that refer to no object are kept as no-ground cases, which teach the existence decision and penalize hallucinated boxes. \textit{4) Consensus review:} Two reviewers independently check box tightness, category correctness, question naturalness, and answer-evidence agreement, and require every numeric answer to equal the cardinality of its box set. Only items that pass both reviews are admitted. Six annotators, who passed a qualification test and a training round of 100 seeded examples, perform annotation and review.

\textit{Annotation quality:} Two adjudicators who did not take part in the annotation re-annotate a stratified random sample of 5,000 items, blind to the released labels and to each other. After reconciliation, the released labels are confirmed for 98.6\% of the answers and 97.8\% of the evidence sets, and count-box agreement holds for 99.2\% of the counting items. Inter-adjudicator agreement before reconciliation reaches a Cohen's $\kappa$ of 0.980. An additional MLLM audit and its calibration are reported in Appendix~\ref{A-app:dataset}.

\begin{table}[!t]
\centering
\caption{Composition of the RoadSceneVQA-G splits. \textbf{No-Grd.}: questions without evidence boxes. \textbf{Grd.}: questions with at least one evidence box. \textbf{Multi-Box}: share of grounded questions with at least two boxes. \textbf{Avg.\ Box}: mean number of evidence boxes per grounded question.}
\label{tab:ds-stats}
\vspace{-3mm}
\setlength{\tabcolsep}{3pt}
\begin{tabular*}{\columnwidth}{@{\extracolsep{\fill}}lccccc@{}}
\toprule
\textbf{Split} & \textbf{Items} & \textbf{No-Grd.} & \textbf{Grd.} & \textbf{Multi-Box} & \textbf{Avg. Box} \\
\midrule
Train & 30{,}059 & 59.6\% & 40.4\% & 32.4\% & 2.06 \\
Test  & 4{,}677  & 15.8\% & 84.2\% & 40.1\% & 1.98 \\
\bottomrule
\end{tabular*}

\end{table}

\begin{table}[!t]
\centering
\caption{Capability taxonomy of RoadSceneVQA-G.}
\label{tab:qtype}
\vspace{-3mm}
\setlength{\tabcolsep}{3pt}
\begin{tabular*}{\columnwidth}{@{\extracolsep{\fill}}llcc@{}}
\toprule
\textbf{Tier} & \textbf{Category} & \textbf{Train} & \textbf{Test} \\
\midrule
\multirow{5}{*}{Perception}
 & Existence              & 24.7 & 11.0 \\
 & Enumeration            & 8.7  & 27.8 \\
 & Attribute recognition  & 41.5 & 31.0 \\
 & Object identification  & 2.4  & 9.3  \\
 & Scene inference        & 14.5 & 3.9  \\
\midrule
\multirow{4}{*}{Reasoning}
 & Spatial reasoning      & 4.1  & 9.8  \\
 & Behavioral reasoning   & 1.2  & 5.3  \\
 & Normative reasoning    & 2.8  & 1.6  \\
 & Causal reasoning       & 0.1  & 0.2  \\
\midrule
\multicolumn{2}{l}{Reasoning share} & 8.2 & 16.9 \\
\bottomrule
\end{tabular*}
\end{table}

\subsection{Dataset Statistics}
\label{sec:statistics}

\textit{Scale and grounding:} RoadSceneVQA-G contains 34,736 QA pairs, 30,059 for training and 4,677 for testing, with essentially one reviewed question per frame (Table~\ref{tab:ds-stats}). Grounding is required for 40.4\% of the training and 84.2\% of the test questions, and 32.4\% and 40.1\% of these require multi-object localization. The 32,831 evidence boxes are small and spread across the scene (Fig.~\ref{fig:dataset_1}), with a median image-area coverage of 0.46\% (train) and 1.00\% (test). The number of boxes per question reaches 55, and 1.9\% of the questions require enumerating seven or more objects.

\textit{Train-to-test shift:} The test split shifts toward the categories in which answer-evidence consistency is most demanding (Table~\ref{tab:qtype}). Enumeration rises from 8.7\% to 27.8\%, the share of reasoning questions doubles from 8.2\% to 16.9\%, and numeric answers rise from 4.4\% to 24.7\%. A model that learns to under-enumerate on the single-object-heavy training split is therefore penalized sharply at test time.

\vspace{-2mm}
\subsection{Evaluation Metrics}
\label{sec:metrics}

The Answer-Grounding Consistency (AGC) suite evaluates the grounding decision, the enumeration, the answer-evidence agreement, and their joint correctness. For each of the $N$ test samples, a model receives an image $I$ and a question $q$ and produces an answer $\hat{a}$ with a possibly empty box set $\hat{B}=\{\hat{b}_1,\dots,\hat{b}_m\}$. The ground truth is $(a^\star,B^\star)$ with $B^\star=\{b^\star_1,\dots,b^\star_n\}$, and $\nu(a)\in\mathbb{N}\cup\{\varnothing\}$ denotes the integer count stated in answer $a$ ($\varnothing$ if none).

\textit{Grounding-decision accuracy (GDA)} scores whether the model decides correctly to localize at all, and \textit{count accuracy (CntAcc)} checks the cardinality of the enumerated set:
\begin{equation}
\begin{aligned}
\mathrm{GDA}&=\frac{1}{N}\sum_{i=1}^{N}\mathbf{1}\!\left[\mathbf{1}[\hat{B}_i\neq\varnothing]=\mathbf{1}[B^\star_i\neq\varnothing]\right],\\
\mathrm{CntAcc}&=\frac{1}{N}\sum_{i=1}^{N}\mathbf{1}\!\left[|\hat{B}_i|=|B^\star_i|\right].
\end{aligned}
\label{eq:gda}
\end{equation}

\textit{Say-point consistency (SPC)}, the central metric, measures whether the model says what it points to. It builds on the say-point mismatch indicator
\begin{equation}
\mathrm{SPM}(\hat{a},\hat{B})=\mathbf{1}\!\left[\nu(\hat{a})\neq\varnothing\;\wedge\;\nu(\hat{a})\neq|\hat{B}|\right],
\label{eq:spm}
\end{equation}
and is its complement over the set $\mathcal{C}=\{i\mid\nu(\hat{a}_i)\neq\varnothing\}$ of predictions that state a count:
\begin{equation}
\mathrm{SPC}=\frac{1}{|\mathcal{C}|}\sum_{i\in\mathcal{C}}\mathbf{1}\!\left[\nu(\hat{a}_i)=|\hat{B}_i|\right].
\label{eq:spc}
\end{equation}
SPC measures agreement between the explicitly stated integer count and the number of predicted boxes. It is computed over predictions with an extractable count, with the number of eligible predictions reported alongside the score.

\textit{Grounded-count accuracy (GCA)} credits a prediction only when the answer and the evidence are both correct:
\begin{equation}
\mathrm{GCA}=\frac{1}{N}\sum_{i=1}^{N}\mathbf{1}\!\left[\mathrm{ans\_ok}(\hat{a}_i,a^\star_i)\wedge F1_{0.5}(\hat{B}_i,B^\star_i)\ge 0.5\right],
\label{eq:gca}
\end{equation}
where $\mathrm{ans\_ok}$ is a lexical answer-matching predicate with a lenient (GCA-l) and a strict (GCA-s) variant (Appendix~\ref{A-app:metrics}). Because both variants are lexical, GCA is a conservative lower bound on joint correctness.

\textit{Grounding quality:} Predicted and ground-truth boxes are matched by a max-IoU Hungarian assignment~\cite{yu2016modeling,shen2025vlm,bai2025univg}. A matched pair is correct when its IoU is at least 0.5. We report precision, recall, and F1 under this matching, the COCO-style AP$_{50:95}$~\cite{lin2014microsoft}, and the recall on samples with seven or more ground-truth objects (R@7+).

\begin{figure*}[!t]
\centering
\includegraphics[width=\textwidth]{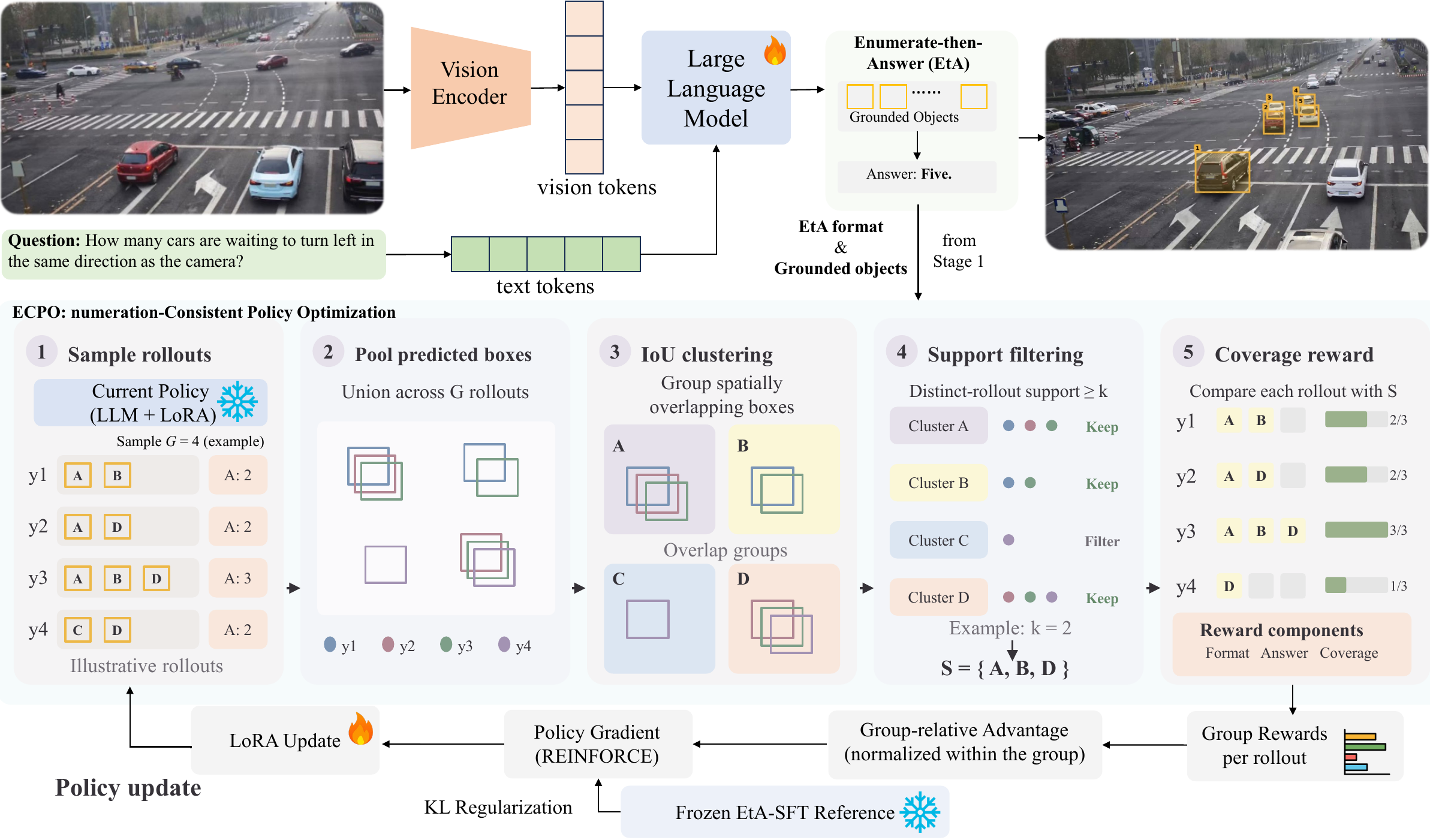}
\vspace{-5mm}
\caption{Overview of the proposed framework. \textbf{Stage~1 (top): Enumerate-then-Answer (EtA)} serializes the grounded objects before the answer, so the answer is read out from the enumerated evidence and an empty box set encodes the decision not to ground. \textbf{Stage~2 (bottom): Enumeration-Consistent Policy Optimization (ECPO)} samples $G$ rollouts from the current policy, pools and clusters their predicted boxes by IoU, keeps the clusters supported by at least $k$ distinct rollouts as the consensus set, and rewards each rollout by its coverage of this set. Group-relative advantages update the policy with KL regularization toward the frozen EtA-SFT reference.}
\label{fig:main_figure}
\end{figure*}

\vspace{-3mm}
\section{Method}
\label{sec:method}

Our framework improves answer-grounding alignment and enumeration completeness through two stages (Fig.~\ref{fig:main_figure}). EtA generates the answer conditioned on the predicted boxes (Section~\ref{sec:eta}), while ECPO uses multi-rollout consensus to improve object coverage without per-box supervision (Section~\ref{sec:ecpo}).

\subsection{Enumerate-then-Answer}
\label{sec:eta}

\textit{Factorization:} Given an image $I$ and a query $q$, a grounded MLLM models the joint distribution of the answer $\hat{a}$ and the evidence set $\hat{B}$, and existing models factorize it in the answer-then-ground order:
\begin{equation}
p(\hat{a},\hat{B}\mid I,q)=p(\hat{a}\mid I,q)\,p(\hat{B}\mid I,q,\hat{a}).
\label{eq:atg}
\end{equation}
The answer tokens, including any count $\nu(\hat{a})$, are sampled before localization, so the model can state quantities that no visual evidence supports. EtA inverts the order of the conditional factors:
\begin{equation}
p_\theta(\hat{a},\hat{B}\mid I,q)=p_\theta(\hat{B}\mid I,q)\,p_\theta(\hat{a}\mid I,q,\hat{B}).
\label{eq:eta}
\end{equation}

\textit{Serialization:} Both factors are parameterized by one autoregressive decoder over a unified token sequence $y=[\,\hat{b}_1,\dots,\hat{b}_m,\hat{a}\,]$. Each box $\hat{b}_j=(x_1,y_1,x_2,y_2)$ is written as integer coordinates normalized to $[0,1000]$. Boxes are sorted by their top-left corner $(y_1,x_1)$, which gives a canonical order and stabilizes decoding in cluttered scenes. An empty list encodes $\hat{B}=\varnothing$, the explicit decision not to ground, so grounded and non-grounded questions share one output format without a routing module. The decoder is trained with the cross-entropy loss:
\begin{equation}
\mathcal{L}_{\mathrm{SFT}}(\theta)=-\sum_{t=1}^{|y|}\log p_\theta\bigl(y_t\mid y_{<t},I,q\bigr).
\label{eq:sft}
\end{equation}

Since the answer tokens follow all box tokens, the stated count is conditioned on the complete enumerated set.

\textit{Theoretical analysis:} We characterize the representation capacity of the two factorizations with two properties of a policy on an input $(I,q)$. \emph{Grounding correctness} holds when $\hat{B}=B^\star$, and \emph{consistency} holds when $\mathrm{SPM}(\hat{a},\hat{B})=0$ in Eq. \ref{eq:spm}.

\begin{proposition}[Representability and trade-off]
\label{prop:eta}
(i) Under the EtA factorization (Eq. \ref{eq:eta}), there exists a policy that is deterministically consistent and grounding-correct on every input.
(ii) Under the answer-then-ground factorization (Eq. \ref{eq:atg}), every policy satisfies, for every $(I,q)$ with $n=|B^\star|$,
\begin{multline}
\Pr[\mathrm{SPM}=1\mid I,q]\;\ge\;1-\Pr[\nu(\hat{a})=\varnothing\mid I,q]\\
-\Pr[\nu(\hat{a})=n\mid I,q]-\Pr[|\hat{B}|\neq n\mid I,q].
\label{eq:prop1}
\end{multline}
A policy that enumerates the correct cardinality and always states a count therefore satisfies $\Pr[\mathrm{SPM}=1\mid I,q]=\Pr[\nu(\hat{a})\neq n\mid I,q]$.
\end{proposition}

\begin{proof}
(i) Let the enumeration factor place $\hat{B}=B^\star$ and let the answer factor implement the deterministic readout $\nu(\hat{a})=|\hat{B}|$. The policy is grounding-correct by construction, and $\nu(\hat{a})=|\hat{B}|$ gives $\mathrm{SPM}(\hat{a},\hat{B})=0$ on every input.

(ii) Conditioning on $(I,q)$, $\mathrm{SPM}=0$ holds exactly when no count is stated or the stated count equals $|\hat{B}|$, so $\Pr[\mathrm{SPM}=0]=\Pr[\nu(\hat{a})=\varnothing]+\Pr[\nu(\hat{a})\neq\varnothing\wedge\nu(\hat{a})=|\hat{B}|]$. If a stated count equals $|\hat{B}|$ but differs from $n$, then $|\hat{B}|\neq n$. Hence $\{\nu(\hat{a})\neq\varnothing\wedge\nu(\hat{a})=|\hat{B}|\}\subseteq\{\nu(\hat{a})=n\}\cup\{|\hat{B}|\neq n\}$, and the union bound followed by taking complements yields (Eq. \ref{eq:prop1}). If the policy always states a count and $|\hat{B}|=n$, then $\mathrm{SPM}=0$ is equivalent to $\nu(\hat{a})=n$, and equality holds.
\end{proof}

Under Eq. \ref{eq:atg}, the answer is sampled from $p(\hat{a}\mid I,q)$, which does not observe $\hat{B}$. The term $\Pr[\nu(\hat{a})=n\mid I,q]$ is therefore the accuracy of counting without localization. A strong answer-then-ground model could tighten the bound by learning to count directly from the image, but it must do so without access to its own boxes, which couples consistency to an auxiliary counting problem. EtA removes this coupling by construction and makes answer-evidence agreement the structural default instead of a learned behavior.

\subsection{Enumeration-Consistent Policy Optimization}
\label{sec:ecpo}

EtA removes the structural source of say-point mismatch, but it does not guarantee that the enumerated set is complete, and omissions are most likely in dense scenes with many small objects. Plain GRPO~\cite{shao2024deepseekmath} could not correct this reliably. When all rollouts omit the same objects, a reward computed per rollout has no variance on these objects and the gradient carries no signal for them. Under-enumeration errors across rollouts are, however, only partially correlated, so the union of their box sets recovers more true targets than any single sample. ECPO turns this coverage into pseudo-labels.

\textit{Consensus set construction:} Given a prompt $(I,q)$, we sample $G$ rollouts $\{y^{(g)}\}_{g=1}^{G}\sim\pi_\theta$ and pool all predicted boxes into $\mathcal{U}=\bigcup_{g}\hat{B}^{(g)}$. The boxes are clustered greedily, and a box $b$ joins cluster $c$ when $\mathrm{IoU}(b,\bar{c})\ge 0.5$, where $\bar{c}$ is the representative box of $c$. Let $\mathrm{supp}(\bar{c})$ be the set of distinct rollouts that contribute a box to $c$. To suppress false positives, the consensus set keeps only the clusters supported by at least $k=\max(2,\lceil 0.25G\rceil)$ rollouts:
\begin{equation}
\tilde{S}=\{\bar{c}\mid|\mathrm{supp}(\bar{c})|\ge k\}.
\label{eq:consensus}
\end{equation}

\textit{Reward design:} The consensus set serves as a pseudo-recall teacher. Each rollout receives a coverage reward
\begin{equation}
r_{\mathrm{grp}}^{(g)}=\frac{|\mathcal{M}_{0.5}(\hat{B}^{(g)},\tilde{S})|}{|\tilde{S}|},
\label{eq:rgrp}
\end{equation}
where $\mathcal{M}_{0.5}$ denotes the Hungarian matching at an IoU of 0.5. A rollout that recovers an object found by other rollouts is rewarded even if the majority missed it, which supplies the signal that plain GRPO lacks. The total reward is
\begin{equation}
r^{(g)}=\lambda_{\mathrm{grp}}r_{\mathrm{grp}}^{(g)}+\lambda_{\mathrm{fmt}}r_{\mathrm{fmt}}^{(g)}+\lambda_{\mathrm{ans}}r_{\mathrm{ans}}^{(g)},
\label{eq:reward}
\end{equation}
where $r_{\mathrm{fmt}}$ is a binary check that the output parses into the EtA format and $r_{\mathrm{ans}}$ is the token-overlap agreement with the reference answer. No term uses ground-truth boxes.

\begin{algorithm}[!t]
\caption{One ECPO update}
\label{alg:ecpo}
\begin{algorithmic}[1]
\Require Policy $\pi_\theta$ initialized from EtA-SFT, frozen reference $\pi_{\mathrm{ref}}$, prompt batch $\mathcal{D}$, group size $G$, support threshold $k$, reward weights $\lambda_{\mathrm{grp}},\lambda_{\mathrm{fmt}},\lambda_{\mathrm{ans}}$, KL weight $\beta$
\For{each $(I,q,a^\star)\in\mathcal{D}$}
  \State Sample rollouts $y^{(1)},\dots,y^{(G)}\sim\pi_\theta(\cdot\mid I,q)$ and parse each into $(\hat{B}^{(g)},\hat{a}^{(g)})$
  \State $\mathcal{U}\gets\bigcup_{g}\hat{B}^{(g)}$, \ $\mathcal{K}\gets\varnothing$ \Comment{pool boxes}
  \For{each box $b\in\mathcal{U}$}
    \If{$\exists\, c\in\mathcal{K}$ with $\mathrm{IoU}(b,\bar{c})\ge 0.5$}
      \State add $b$ to $c$
    \Else
      \State $\mathcal{K}\gets\mathcal{K}\cup\{\{b\}\}$
    \EndIf
  \EndFor
  \State $\tilde{S}\gets\{\bar{c}\mid c\in\mathcal{K},\ |\mathrm{supp}(\bar{c})|\ge k\}$ \Comment{Eq.~\ref{eq:consensus}}
  \For{$g=1,\dots,G$}
    \State compute $r_{\mathrm{grp}}^{(g)}$ by Eq.\ref{eq:rgrp}, $r_{\mathrm{fmt}}^{(g)}$, and $r_{\mathrm{ans}}^{(g)}$ against $a^\star$
    \State $r^{(g)}\gets\lambda_{\mathrm{grp}}r_{\mathrm{grp}}^{(g)}+\lambda_{\mathrm{fmt}}r_{\mathrm{fmt}}^{(g)}+\lambda_{\mathrm{ans}}r_{\mathrm{ans}}^{(g)}$
  \EndFor
  \State $A^{(g)}\gets(r^{(g)}-\mu)/(\sigma+\epsilon)$ for all $g$ \Comment{normalize}
\EndFor
\State Update $\theta$ by a gradient step on $\mathcal{L}_{\mathrm{ECPO}}$ in Eq. \ref{eq:ecpo}
\end{algorithmic}
\end{algorithm}

\textit{Policy optimization:} Algorithm~\ref{alg:ecpo} summarizes one ECPO update. Rewards are normalized within the group as $A^{(g)}=(r^{(g)}-\mu)/(\sigma+\epsilon)$, where $\mu$ and $\sigma$ are the group mean and standard deviation. We minimize the group-relative loss with KL regularization toward the frozen EtA-SFT reference $\pi_{\mathrm{ref}}$:
\begin{multline}
\mathcal{L}_{\mathrm{ECPO}}(\theta)=-\mathbb{E}\Bigl[\frac{1}{|y^{(g)}|}\sum_{t}A^{(g)}\log\pi_\theta\bigl(y_t^{(g)}\mid y_{<t}^{(g)},I,q\bigr)\Bigr]\\
+\beta\,\mathbb{E}\bigl[\mathrm{KL}(\pi_\theta\,\|\,\pi_{\mathrm{ref}})\bigr].
\label{eq:ecpo}
\end{multline}

\begin{proposition}[Recall amplification]
\label{prop:ecpo}
Assume each target $b_j^\star\in B^\star$ is detected independently in each rollout with probability at least $p$. Then the union $\mathcal{U}$ of $G$ rollouts satisfies
\begin{equation}
\mathbb{E}[R(\mathcal{U})]\ge 1-(1-p)^G> p\quad\text{for } G\ge 2,\; p\in(0,1).
\end{equation}
\end{proposition}

\begin{proof}
The union misses $b_j^\star$ only when every rollout misses it, which happens with probability at most $(1-p)^G$. Recall is the fraction of recovered ground-truth objects, so linearity of expectation over the $n$ objects gives $\mathbb{E}[R(\mathcal{U})]\ge 1-(1-p)^G$. Since $(1-p)^G<1-p$ for $G\ge 2$, the bound exceeds $p$.
\end{proof}

The margin $1-(1-p)^G-p$ decreases in $p$, so the amplification is largest where single-rollout recall is lowest, i.e., for small objects in dense scenes. When rollouts are positively correlated, the bound becomes optimistic, but the union still contains every object that at least one rollout detects, so its recall never falls below that of any single rollout.

\section{Experiments}
\label{sec:experiments}

\subsection{Datasets and Evaluation Metrics}
\label{sec:settings}

We evaluate on three benchmarks. \textbf{1) RoadSceneVQA-G} (Section~\ref{sec:dataset}) is evaluated with the AGC suite and the grounding metrics of Section~\ref{sec:metrics}, together with METEOR~\cite{meteor} and BERTScore on the answer span. GCA denotes GCA-l. \textbf{2) gRefCOCO}~\cite{ding2026grex} contains 259,859 referring expressions over 61,316 objects in 19,994 images, including 90,064 multi-object and 34,537 no-object expressions. We follow the UNC partition~\cite{yu2016modeling} and report Pr@F1=1, the fraction of samples with a perfect box-set F1, and N-acc., the accuracy of rejecting no-target expressions. \textbf{3) CARPK}~\cite{hsieh2017drone} is a dense parking-lot counting benchmark with real car boxes, whose 3,716 training images we convert to the EtA format, and we report the count MAE, box recall at an IoU of 0.5, and SPC.

\subsection{Implementation Details}
\label{sec:impl}

\textit{Baselines:} We compare three groups of methods. General-purpose MLLMs include InternVL3.5~\cite{wang2025internvl3}, Qwen3.5~\cite{qwen2026qwen35}, LLaVA-OneVision-2~\cite{an2026llavaonevision2}, and MiniCPM-V~4.6~\cite{yao2024minicpm}, evaluated zero-shot and after standard LoRA fine-tuning. Visual grounding (VG) specialists include DINO-X~\cite{ren2024dino}, queried through its official API, and GroundingDINO-T~\cite{liu2024grounding} and YOLO-World-v2-L~\cite{cheng2024yolo}, fine-tuned on our training split with their official configurations for 12 and 24 epochs, respectively. The specialists receive QA-derived referring expressions and are evaluated only as box predictors. Grounded MLLMs include Ferret~\cite{you2024ferret}, Shikra~\cite{chen2023shikra}, and Groma~\cite{ma2024groma}, re-implemented with Qwen3 as the common language backbone and trained with their native box syntax and official hyperparameters. For the reinforcement learning stage, vanilla GRPO~\cite{shao2024deepseekmath}, Perception-R1~\cite{yu2026perception}, and VisionReasoner~\cite{liu2026visionreasoner} use the reward designs of their official implementations and start from the same EtA-SFT checkpoint as ECPO.

\textit{EtA-SFT:} We fine-tune Qwen3-VL-8B-Instruct~\cite{bai2025qwen3} with LoRA of rank 64 and $\alpha=16$ on the language model and the multimodal projector, keeping the vision encoder frozen. Training runs for three epochs in bf16 with a learning rate of $2\times10^{-4}$, cosine decay, 3\% warm-up, and an effective batch size of 36. Images are capped at 602,112 pixels (about 768 visual tokens) to preserve small distant objects, and sequences at 2,048 tokens. The standard SFT baseline uses the same configuration with the answer-then-ground order.

\textit{ECPO:} Starting from the merged EtA-SFT checkpoint, we train a LoRA adapter of rank 32 and $\alpha=64$ for one epoch with a learning rate of $10^{-6}$. We sample $G=6$ rollouts per prompt, which gives $k=2$ in Eq.\ref{eq:consensus}, and set $\lambda_{\mathrm{grp}}=0.6$, $\lambda_{\mathrm{fmt}}=0.1$, and $\lambda_{\mathrm{ans}}=0.3$. Completions are capped at 640 tokens and images at 401,408 pixels. The reference policy is recovered by disabling the adapter, and other settings follow GRPO~\cite{shao2024deepseekmath}. All experiments run on six NVIDIA RTX~3090 GPUs, and SFT uses DeepSpeed ZeRO-2.

\textit{Transfer settings:} On gRefCOCO (209,344 training expressions), InternVL3-1B~\cite{zhu2025internvl3} uses the same adapter and batch size for three epochs at a learning rate of $10^{-4}$. On CARPK, the sequence cap is raised to 3,072 tokens (Appendix~\ref{A-app:impl}).

\begin{table*}[!t]
\centering
\caption{Overall comparison on the RoadSceneVQA-G test split. Ans.: whether the method produces a VQA answer. N/A: the metric requires an answer, which VG specialists do not produce. AP denotes AP$_{50:95}$, and GCA denotes GCA-l. Best results are in bold, and further metrics are in Appendix~\ref{A-app:full_baselines}.}
\label{tab:roadscene_overview}
\vspace{-3mm}
\setlength{\tabcolsep}{9.9pt}
\begin{tabular}{lc cc cc cccc}
\toprule
 & & \multicolumn{2}{c}{\textbf{Answer quality}} & \multicolumn{2}{c}{\textbf{Grounding}} & \multicolumn{4}{c}{\textbf{AGC (decision / consistency / joint)}} \\
\cmidrule(lr){3-4}\cmidrule(lr){5-6}\cmidrule(lr){7-10}
\textbf{Method} & \textbf{Ans.} & \textbf{METEOR} & \textbf{BERTScore} & \textbf{F1} & \textbf{AP} & \textbf{GDA} & \textbf{SPC} & \textbf{CntAcc} & \textbf{GCA} \\
\midrule
\multicolumn{10}{l}{\textit{Zero-shot MLLMs: answer-capable but weak or unstable box enumeration}} \\
InternVL3.5-8B-HF~\cite{wang2025internvl3} & \cmark & 22.0 & 89.4 & 36.9 & 13.0 & 60.0 & 60.0 & 30.3 & 4.3 \\
Qwen3.5-27B (thinking)~\cite{qwen2026qwen35} & \cmark & 20.1 & 86.1 & 49.1 & 26.1 & 79.5 & 82.3 & 49.8 & 13.2 \\
Qwen3.5-27B (non-thinking)~\cite{qwen2026qwen35} & \cmark & 19.7 & 88.8 & 64.8 & 34.5 & 78.7 & 80.1 & 46.9 & 10.8 \\
LLaVA-OneVision-2-8B~\cite{an2026llavaonevision2} & \cmark & 15.3 & 90.5 & 8.7 & 1.1 & 32.6 & 20.7 & 12.6 & 3.0 \\
\midrule
\multicolumn{10}{l}{\textit{LoRA-SFT MLLMs: language quality improves, but box emission regresses without the EtA structure}} \\
InternVL3.5-8B~\cite{wang2025internvl3} & \cmark & 24.3 & 91.8 & 14.6 & 9.1 & 25.5 & 58.7 & 50.1 & 6.5 \\
MiniCPM-V 4.6~\cite{yao2024minicpm} & \cmark & 23.4 & 91.8 & 13.8 & 2.4 & 88.2 & 61.7 & 51.3 & 6.9 \\
\midrule
\multicolumn{10}{l}{\textit{VG specialists: strong localization after adaptation but no VQA answer}} \\
DINO-X~\cite{ren2024dino} & \xmark & N/A & N/A & 40.2 & 27.9 & 76.5 & N/A & 40.9 & N/A \\
GroundingDINO-T~\cite{liu2024grounding} & \xmark & N/A & N/A & 72.7 & \textbf{62.7} & 86.3 & N/A & 56.1 & N/A \\
YOLO-World-v2-L~\cite{cheng2024yolo} & \xmark & N/A & N/A & 19.8 & 3.0 & 51.4 & N/A & 10.4 & N/A \\
\midrule
\multicolumn{10}{l}{\textit{Grounded MLLMs}} \\
Ferret (Qwen3)~\cite{you2024ferret} & \cmark & 17.6 & \textbf{92.8} & 15.7 & 5.2 & 36.2 & 23.3 & 31.6 & 19.6 \\
Shikra (Qwen3)~\cite{chen2023shikra} & \cmark & 20.0 & 92.4 & 18.3 & 14.3 & 42.3 & 28.1 & 37.8 & 22.5 \\
Groma (Qwen3)~\cite{ma2024groma} & \cmark & 25.1 & 92.6 & 61.8 & 31.2 & 85.4 & 87.8 & 71.6 & 46.1 \\
Qwen3-VL-8B (standard SFT) & \cmark & 26.7 & 92.5 & 52.2 & 31.9 & 88.8 & 26.6 & 52.1 & 46.3 \\
\rowcolor{gray!30} \textbf{Qwen3-VL-8B (EtA-SFT)} & \cmark & 26.6 & 92.6 & 73.0 & 54.4 & 88.0 & 93.7 & 65.9 & 51.2 \\
\rowcolor{gray!30} \textbf{Qwen3-VL-8B (EtA-SFT + ECPO)} & \cmark & \textbf{26.8} & 92.7 & \textbf{75.6} & 60.8 & \textbf{95.6} & \textbf{97.9} & \textbf{72.7} & \textbf{55.5} \\
\bottomrule
\end{tabular}

\end{table*}

\subsection{Quantitative Results}
\label{sec:main_results}

\textbf{Overall comparison.} Table~\ref{tab:roadscene_overview} shows that the three baseline families have complementary weaknesses. Zero-shot MLLMs answer fluently but enumerate unreliably, and the strongest reaches only 64.8 F1. Standard LoRA SFT improves language quality but often suppresses box emission. VG specialists localize well after fine-tuning but produce no answer to verify. With the same 8B backbone, EtA-SFT reaches 73.0 F1, 8.2 points above the best zero-shot MLLM, while remaining answer-capable. ECPO further reaches 75.6 F1, 95.6 GDA, and 72.7 CntAcc, 2.9 points above fine-tuned GroundingDINO-T, while producing an answer that its boxes witness.

\begin{table}[!t]
\centering
\caption{Generalized referring expression comprehension on gRefCOCO. Each entry reports Pr@F1=1 / N-acc.\ (no-target accuracy). }
\label{tab:grefcoco}
\vspace{-3mm}
\setlength{\tabcolsep}{20pt}
\begin{tabular}{lcc}
\toprule
\textbf{Method} & \textbf{testA} & \textbf{testB} \\
\midrule
VLT~\cite{ding2021vision} & 40.2 / 34.1 & 30.2 / 32.5 \\
MDETR~\cite{kamath2021mdetr} & 50.0 / 34.5 & 36.5 / 31.0 \\
UNINEXT~\cite{yan2023universal} & 46.4 / 49.3 & 42.9 / 48.2 \\
ReLA~\cite{liu2023gres} & 50.4 / 59.0 & 44.6 / 58.4 \\
LLMDet~\cite{fu2025llmdet} & 62.8 / 78.9 & 57.3 / 76.2 \\
\midrule
\multicolumn{3}{l}{\textit{Ours (InternVL3-1B~\cite{zhu2025internvl3})}} \\
\rowcolor{gray!30} \textbf{EtA-SFT}        & 54.7 / 91.2 & 46.4 / 85.7 \\
\rowcolor{gray!30} \textbf{EtA-SFT + ECPO} & \textbf{65.6} / \textbf{92.9} & \textbf{59.3} / \textbf{86.6} \\
\bottomrule
\end{tabular}
\end{table}

\textbf{Transfer to gRefCOCO.} gRefCOCO tests the two decisions that EtA learns, enumerating a variable number of targets and abstaining when none exists, in general scenes (Table~\ref{tab:grefcoco}). InternVL3-1B with EtA-SFT and ECPO reaches 65.6 Pr@F1=1 on testA, 2.8 points above LLMDet, and 92.9 N-acc., 14.0 points above LLMDet. ECPO improves Pr@F1=1 over EtA-SFT by 10.9 and 12.9 points on testA and testB, so the consensus reward also helps outside the traffic domain.

\begin{table}[!t]
\centering
\caption{Dense car counting on CARPK. Count MAE is computed per image, and Rec@0.5 is box recall at an IoU of 0.5. }
\label{tab:carpk}
\vspace{-3mm}
\setlength{\tabcolsep}{9pt}
\begin{tabular}{lccc}
\toprule
\textbf{Method} & \textbf{Count MAE} $\downarrow$ & \textbf{Rec@0.5} $\uparrow$ & \textbf{SPC} $\uparrow$ \\
\midrule
MSS/PPM-CNN~\cite{dearruda2022counting} & 4.45 & -- & -- \\
Direct-answer SFT & 2.48 & -- & -- \\
\midrule
\rowcolor{gray!30} \textbf{EtA-SFT} & 2.51 & 82.3 & 78.2 \\
Vanilla GRPO~\cite{shao2024deepseekmath} & 2.53 & \textbf{82.3} & 78.5 \\
\rowcolor{gray!30} \textbf{ECPO} & \textbf{2.46} & \textbf{82.3} & \textbf{81.0} \\
\bottomrule
\end{tabular}

\end{table}

\textbf{Dense counting on CARPK.} CARPK tests whether EtA enumerates real boxes in uniformly dense scenes (Table~\ref{tab:carpk}). EtA-SFT reaches 82.3 box recall with a count MAE of 2.51, close to the 2.48 of direct-answer SFT, which produces no evidence. ECPO attains the lowest MAE (2.46 vs.\ 2.53 for vanilla GRPO) and raises SPC from 78.2 to 81.0.

\vspace{-3mm}
\subsection{Ablation Study}
\label{sec:ablation}

\begin{table}[!t]
\centering
\caption{Consistency controls on the standard SFT (answer-then-ground) checkpoint. Count rewriting replaces each stated count with the number of emitted boxes, leaving the boxes and thus all grounding metrics unchanged. Boxes-first prompting reorders the output at inference time without retraining. }
\label{tab:controls}
\vspace{-3mm}
\setlength{\tabcolsep}{4.5pt}
\begin{tabular}{lcccccc}
\toprule
\textbf{Configuration} & \textbf{SPC} & \textbf{Prec} & \textbf{Rec} & \textbf{F1} & \textbf{CntAcc} & \textbf{GCA} \\
\midrule
Standard SFT (raw) & 26.6 & 83.8 & 37.9 & 52.2 & 52.1 & 46.3 \\
\quad + count rewriting & 99.9 & 83.8 & 37.9 & 52.2 & 52.1 & 40.2 \\
\quad + boxes-first prompting & 2.9 & 9.6 & 78.6 & 17.1 & 18.5 & 12.4 \\
\midrule
\rowcolor{gray!30} \textbf{EtA-SFT} & 93.7 & \textbf{82.6} & 65.4 & 73.0 & 65.9 & 51.2 \\
\rowcolor{gray!30} \textbf{EtA-SFT + ECPO} & \textbf{97.9} & 81.0 & \textbf{70.9} & \textbf{75.6} & \textbf{72.7} & \textbf{55.5} \\
\bottomrule
\end{tabular}

\end{table}

\begin{figure*}[!t]
\centering
\includegraphics[width=\textwidth]{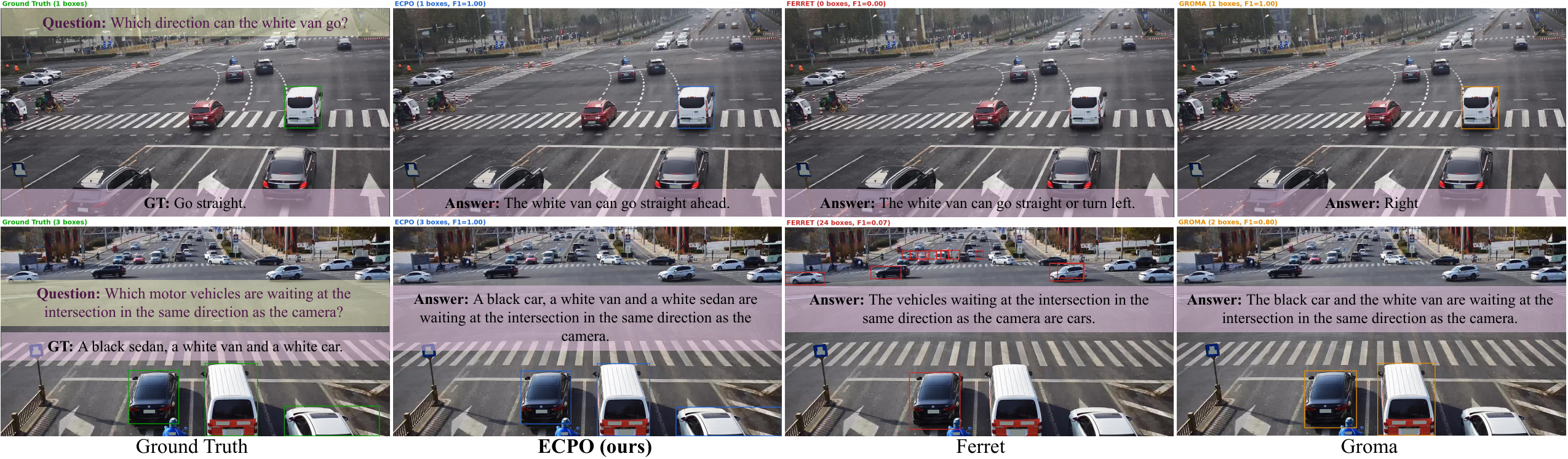}
\vspace{-7mm}
\caption{Qualitative comparison on RoadSceneVQA-G. Top: for a direction question, ECPO localizes the white van and answers correctly, Ferret emits no box and a wrong answer, and Groma localizes the van but answers incorrectly. Bottom: for a multi-object question, ECPO enumerates all three waiting vehicles and names them in its answer, Groma misses the white car, and Ferret emits 24 boxes.}
\label{fig:vis_compare}
\end{figure*}

\textbf{Necessity of the EtA structure.} The standard SFT model shares the backbone, data, and training budget with EtA-SFT and differs only in the serialization order. It is say-point consistent on only 26.6\% of the count-stating predictions, compared with 93.7\% for EtA-SFT and 97.9\% after ECPO (Table~\ref{tab:roadscene_overview}). EtA also raises F1 from 52.2 to 73.0 and AP from 31.9 to 54.4, because the answer no longer commits the model to a count that the enumeration must then justify. The gains concentrate where consistency binds. On enumeration questions, ECPO raises SPC from 15.9 to 99.0 and F1 from 33.1 to 71.3 over standard SFT. On existence questions, it raises GDA from 48.1 to 75.0 and CntAcc from 22.4 to 57.0, whereas single-object categories are near saturation for all models (Appendix~\ref{A-app:per_category}).

\textbf{Consistency controls.} Because SPC reads only the model output, two shortcuts could raise it without EtA training (Table~\ref{tab:controls}). Rewriting every stated count of the standard SFT model to its number of emitted boxes raises SPC to 99.9\%, but the boxes are untouched, so F1 is unchanged and GCA drops from 46.3 to 40.2. The agreement of the stated numbers with the true cardinality also falls from 42.7\% to 14.1\%, because the answers now inherit the enumeration errors. Prompting the same model to emit boxes first at inference time produces more than 57,000 false-positive boxes against 6,114 true positives, so precision collapses to 9.6 and SPC to 2.9. Consistency alone is therefore cheap, and EtA matters because it learns the boxes-first structure while improving the boxes themselves.

\begin{table}[!t]
\centering
\caption{Comparison of reinforcement learning methods on RoadSceneVQA-G. All methods start from the same EtA-SFT checkpoint. The first two rows are the SFT references. GT: whether the reward uses per-box ground-truth annotations. R@7+: recall on samples with $\ge 7$ objects.}
\label{tab:main}
\vspace{-3mm}
\setlength{\tabcolsep}{7.5pt}
\begin{tabular}{lccccc}
\toprule
\textbf{Method} & \textbf{GT} & \textbf{F1} & \textbf{GDA} & \textbf{CntAcc} & \textbf{R@7+} \\
\midrule
Standard SFT & -- & 52.2 & 88.8 & 52.1 & 6.7 \\
EtA-SFT      & -- & 73.0 & 88.0 & 65.9 & 34.8 \\
\midrule
Vanilla GRPO~\cite{shao2024deepseekmath} & --     & 73.7 & 95.1 & 72.0 & 36.8 \\
Perception-R1~\cite{yu2026perception} & \cmark & \textbf{75.6} & 95.2 & 72.3 & 37.2 \\
VisionReasoner~\cite{liu2026visionreasoner} & \cmark & \textbf{75.6} & 95.0 & 72.0 & 37.1 \\
\rowcolor{gray!30} \textbf{ECPO (ours)} & \xmark & \textbf{75.6} & \textbf{95.6} & \textbf{72.7} & \textbf{38.1} \\
\bottomrule
\end{tabular}

\end{table}

\begin{table}[!t]
\centering
\caption{Ablation of the ECPO reward terms. Each row modifies the reward of vanilla GRPO. Group-union coverage accounts for the gain. The last row adds a per-box ground-truth recall term and serves as a supervised reference.}
\label{tab:ablation}
\vspace{-3mm}
\setlength{\tabcolsep}{9pt}
\begin{tabular}{lcccc}
\toprule
\textbf{Configuration} & \textbf{F1} & \textbf{Rec} & \textbf{GDA} & \textbf{CntAcc} \\
\midrule
Vanilla GRPO & 73.7 & 69.8 & 95.1 & 72.0 \\
\quad + group-union coverage & 75.6 & 70.9 & 95.6 & 72.7 \\
\quad + SPC only & 75.4 & 70.5 & 94.7 & 71.6 \\
\quad + union + SPC & 75.9 & 71.3 & 95.9 & 73.0 \\
\rowcolor{gray!30} \quad + union + GT recall & \textbf{76.0} & \textbf{71.4} & \textbf{96.1} & \textbf{73.2} \\
\bottomrule
\end{tabular}

\end{table}

\textbf{Reinforcement learning methods.} In Table~\ref{tab:main}, vanilla GRPO reaches 73.7 F1, and ECPO reaches 75.6 F1, the same score as Perception-R1 and VisionReasoner, whose rewards use Hungarian-matched per-box ground truth. ECPO attains the highest R@7+ (38.1). The methods differ in supervision rather than performance, and because per-box annotation is the costliest part of grounded VQA data, reaching the supervised operating point without it is the practical benefit of ECPO.

\textbf{Reward terms.} Table~\ref{tab:ablation} shows that group-union coverage is the operative term. A say-point consistency term alone slightly degrades the grounding decision, and adding it to the union reward changes little, because the EtA serialization already provides this consistency. A per-box ground-truth recall term adds only 0.4 F1, so ECPO omits it.

\textbf{Under-enumeration and consensus quality.} The residual error of EtA-SFT is almost entirely omission. Its recall falls from 77\% on single-object questions to 34.8\% on questions with seven or more objects, while precision stays high. All reinforcement learning variants plateau between 36.8\% and 38.1\% R@7+, with or without per-box supervision. Measurements with $G=6$ rollouts on 2,000 test prompts (Table~\ref{A-tab:consensus} in the appendix) explain this plateau. The union recall of 77.9\% is 18.1 points below the independence prediction of Proposition~\ref{prop:ecpo}, yet nearly twice the single-rollout detection rate of 0.416. The consensus precision rises with support from 53.7\% at $k=2$ to 91.0\% at full support. On the 7+ bucket, however, the teacher retains only 32.1\% of the ground-truth objects, which bounds what the coverage reward can teach. The omitted objects are mostly small and distant, which points to input resolution and training evidence at the SFT stage.

\subsection{Qualitative Analysis}
\label{sec:qualitative}

\begin{figure}[!t]
\centering
\includegraphics[width=\columnwidth]{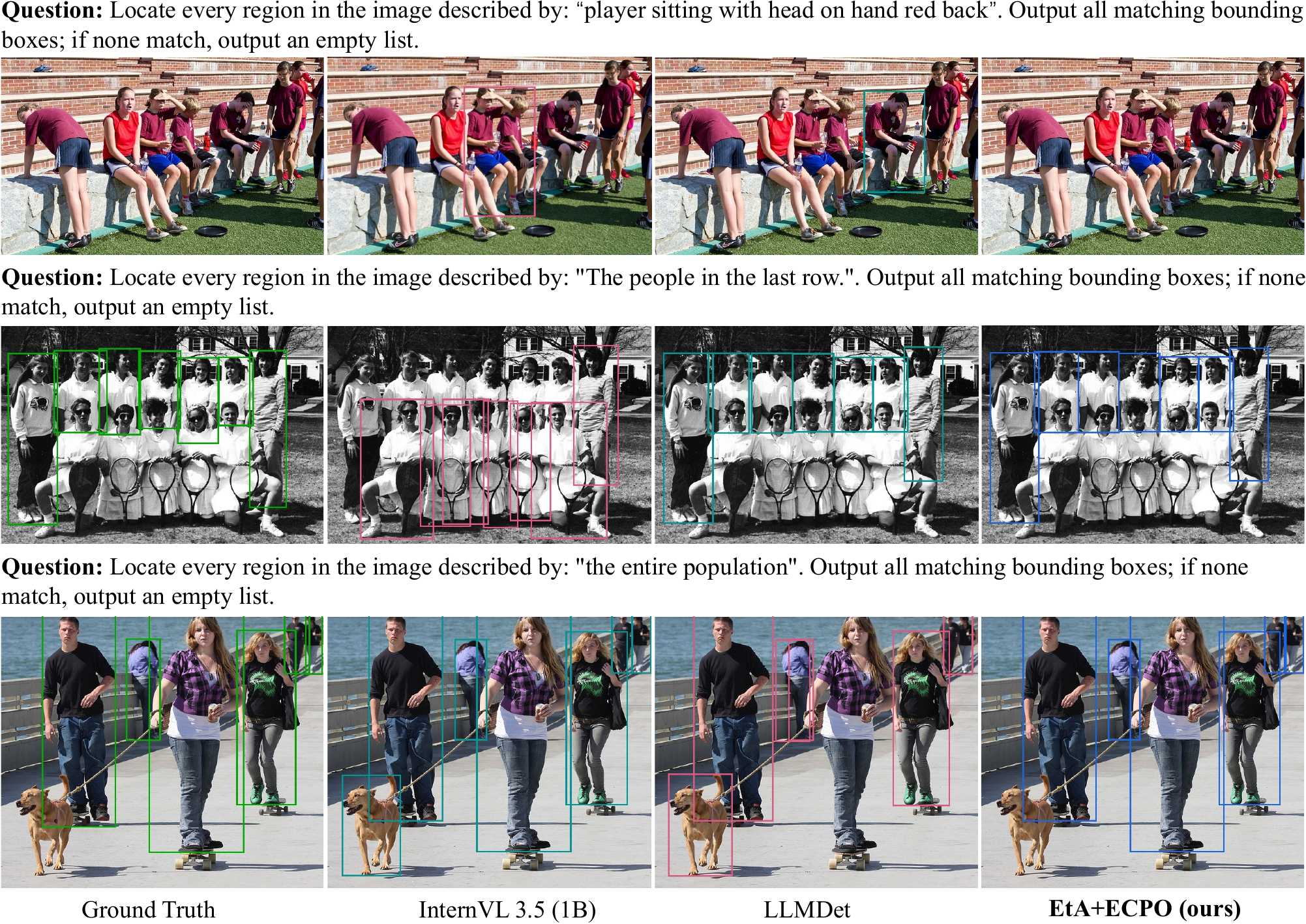}
\vspace{-6mm}
\caption{Qualitative comparison on gRefCOCO. Columns: ground truth, InternVL3.5-1B~\cite{wang2025internvl3}, LLMDet~\cite{fu2025llmdet}, and EtA+ECPO (ours). Top: a no-target expression, for which InternVL3.5-1B and LLMDet each output a spurious box, whereas our model correctly returns an empty set. Middle and bottom: collective multi-target expressions, for which our model enumerates the referred instances with tight boxes, whereas the baselines under-enumerate or leak onto visually similar distractors.}
\label{fig:grefcoco_main}
\end{figure}

\textbf{Roadside scenes.} Fig.~\ref{fig:vis_compare} compares ECPO with Ferret and Groma. Groma localizes the correct van for a direction question but answers wrongly, an error that GCA penalizes and box metrics alone miss. For a multi-object question, ECPO enumerates and names all three waiting vehicles, whereas Groma misses one and Ferret scatters 24 boxes.

\textbf{General scenes.} In Fig.~\ref{fig:grefcoco_main}, InternVL3.5-1B and LLMDet each output a spurious box for a no-target expression, whereas our model returns an empty set and enumerates the referred instances of collective expressions more completely. CARPK examples are given in Appendix~\ref{A-app:qualitative}.

\section{Conclusion}
\label{sec:conclusion}

In this paper, we identified say-point mismatch as a structural flaw inherent to the answer-then-ground formulation and introduced RoadSceneVQA-G along with the AGC evaluation suite to measure it. To address this issue, EtA eliminates this failure mode by generating visual evidence prior to the textual answer, while ECPO recovers dense-scene recall using consensus rewards that operate without box annotations. These findings demonstrate that decoding order is a critical design choice with direct consequences for multimodal alignment, and reversing it requires neither additional supervision nor parameters. In future work, we plan to extend this formulation to video streams and adverse environmental conditions while refining high-resolution encoding for dense scenes. RoadSceneVQA-G and the AGC suite will be publicly released to support further research.

\appendices

\section{Additional Metric Details}
\label{app:metrics}

The AGC metrics (GDA, CntAcc, SPC, and GCA) are defined in Section~\ref{M-sec:metrics} of the main paper. We use the same notation here.

\textit{Grounding quality:} Box matching uses the max-IoU bipartite Hungarian assignment $M_\tau(\hat{B},B^\star)$ at threshold $\tau$, and a matched pair is correct when its IoU reaches $\tau$. Following the standard detection protocol~\cite{lin2014microsoft}, we report precision, recall, and F1 at $\tau=0.5$, together with AP at an IoU of 0.5 and the COCO-style mean over $\mathrm{IoU}\in\{0.5,0.55,\dots,0.95\}$. For dense scenes, we report recall on samples with four to six ground-truth boxes and on those with seven or more, denoted R@4-6 and R@7+.

\textit{Count error:} In addition to CntAcc, we use the mean absolute cardinality error
\begin{equation}
\mathrm{CntMAE}=\frac{1}{N}\sum_{i=1}^{N}\bigl|\,|\hat{B}_i|-|B^\star_i|\,\bigr|.
\label{eq:app_count}
\end{equation}

\textit{Answer matching in GCA:} The strict predicate (GCA-s) applies exact or containment matching after lowercasing. The lenient predicate (GCA-l) first normalizes case, punctuation, and number words, so that \emph{four} and 4 are treated as equal. It then accepts a bare count whenever the stated numbers agree and accepts any other answer under containment or a token-level F1 of at least 0.5. The lenient criterion is lexical and therefore conservative for free-form answers, so GCA is a lower bound on joint correctness, and semantic answer quality is assessed separately with BERTScore. All GCA values in the main paper use the lenient variant.

\textit{SPC ceiling:} The reference annotations attain an SPC of about 92\%, whereas the human re-adjudication verifies count-box agreement on 99.2\% of the audited counting items. The gap comes from approximate quantity words, which pass human verification but not the exact-match SPC rule.

\textit{Text metrics:} BLEU~\cite{papineni2002bleu}, ROUGE-L~\cite{lin2004rouge}, METEOR~\cite{meteor}, and BERTScore are computed on the answer span only, with the box serialization removed. CIDEr~\cite{vedantam2015cider} is reported for completeness in Appendix~\ref{app:full_baselines}.

\section{Additional Implementation Details}
\label{app:impl}

The main training settings are given in Section~\ref{M-sec:impl} of the main paper. This section lists the remaining details.

\textit{Hardware and batching:} All runs use six NVIDIA RTX~3090 GPUs with 24~GB of memory each. For EtA-SFT on RoadSceneVQA-G, the effective batch size of 36 consists of a per-device batch size of 1 and 6 gradient-accumulation steps on six GPUs, and we use gradient checkpointing. The standard SFT baseline shares this configuration exactly.

\textit{ECPO:} ECPO runs for one epoch over the RoadSceneVQA-G training split with the base weights frozen. At the end of training, the KL divergence to the reference policy is approximately $2.5\times10^{-3}$.

\textit{VG specialists:} GroundingDINO-T and YOLO-World-v2-L are evaluated both zero-shot and after fine-tuning. Runtimes, checkpoints, and the number of predicted boxes are listed with the results in Table~\ref{tab:roadscene_vg_specialists}.

\textit{Transfer to gRefCOCO:} The gRefCOCO training split contains 209,344 referring expressions after conversion to the EtA format~\cite{ding2026grex}. We fine-tune both Qwen3-VL-8B and InternVL3-1B with the EtA serialization. Qwen3-VL-8B uses the RoadSceneVQA-G recipe (LoRA rank 64, $\alpha=16$, learning rate $2\times10^{-4}$) for one epoch. InternVL3-1B uses the same adapter configuration with a learning rate of $10^{-4}$, three epochs, and an effective batch size of 36 (per-device batch size 2, gradient accumulation 3, six GPUs). ECPO starts from the corresponding EtA-SFT checkpoint with the settings of the main paper. Evaluation follows the official GREC protocol on the validation, testA, and testB splits.

\textit{Training on CARPK:} The 3,716 CARPK~\cite{hsieh2017drone} training images are converted to the EtA format. Qwen3-VL-8B is trained with the RoadSceneVQA-G recipe (LoRA rank 64, $\alpha=16$, three epochs, learning rate $2\times10^{-4}$, effective batch size 36), except that the sequence cap is raised to 3,072 tokens because parking-lot scenes produce long box lists. ECPO then runs with the RoadSceneVQA-G settings. Direct-answer SFT shares this configuration with answer-only targets.

\section{Additional Dataset Details}
\label{app:dataset}

\begin{figure*}[!t]
\centering
\includegraphics[width=\textwidth]{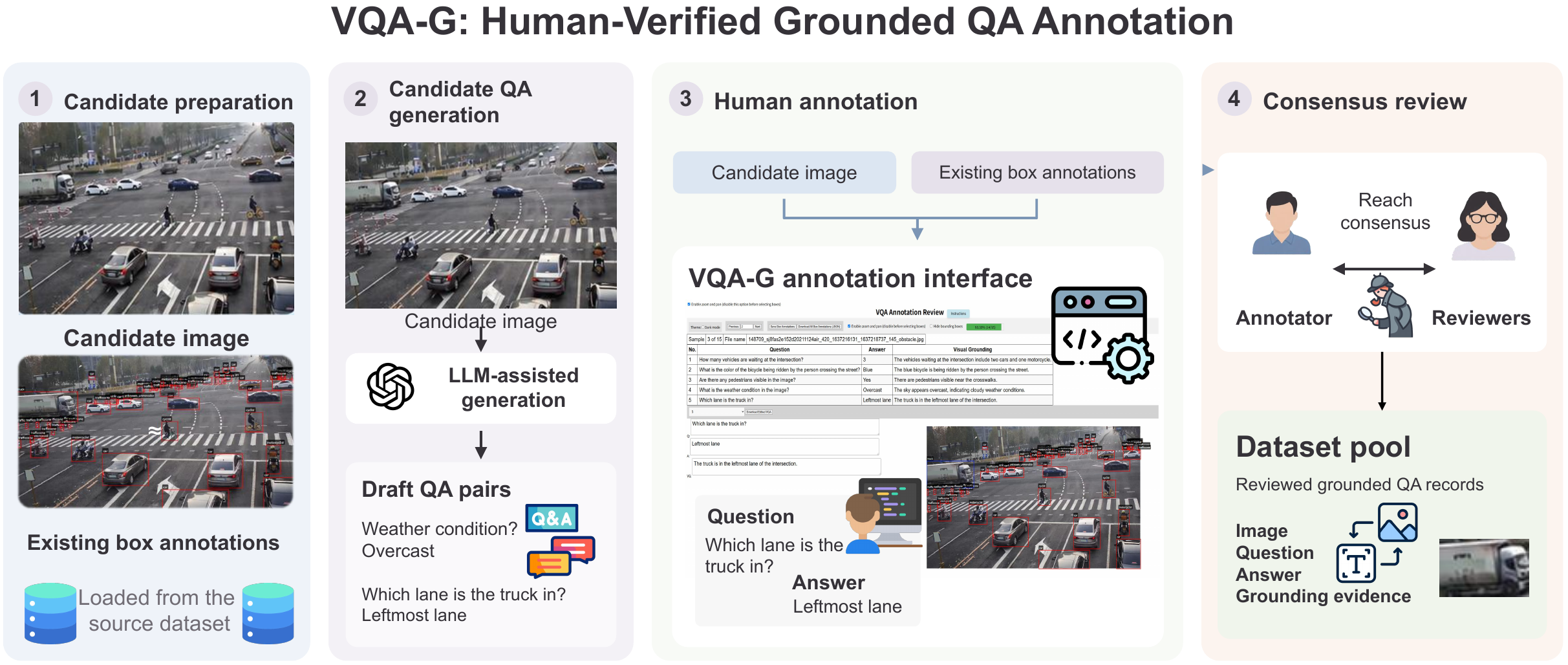}
\vspace{-3.5mm}
\caption{Annotation workflow of RoadSceneVQA-G. (1) Candidate frames and their existing box annotations are loaded from the source dataset. (2) An LLM drafts candidate QA pairs. (3) Annotators edit each question and answer and select the evidence boxes in a dedicated interface. (4) Two reviewers reach consensus before an item enters the dataset.}
\label{fig:annotation_tools}
\end{figure*}

\textit{Annotation workflow:} Fig.~\ref{fig:annotation_tools} illustrates the annotation pipeline described in Section~\ref{M-sec:construction}.

\textit{Taxonomy assignment:} The capability categories in Table~\ref{M-tab:qtype} are assigned by an LLM classifier based on GPT-5~\cite{openai2025gpt5} over all questions rather than by keyword matching. The classifier reads each question and selects one of nine categories, and the resulting labels are used only for analysis. The reasoning share of the test split is twice that of the training split, which is the shift we exploit for diagnosis.

\textit{Independent MLLM audit:} In addition to the consensus review and the human re-adjudication described in Section~\ref{M-sec:construction}, we audit a stratified random sample of 5,000 annotations with a strong MLLM acting as an independent judge. The judge receives the image, the question, the reference answer, and the selected evidence boxes, and it verifies answer correctness, evidence correctness, and count-box agreement. As summarized in Table~\ref{tab:app_audit}, the judge finds no error in any audited item. The judge itself is fallible: the calibration in Appendix~\ref{app:controls_judge} shows that it detects only about half of deliberately corrupted evidence sets. The audit is therefore best read as a conservative floor on top of the human-enforced consistency contract rather than as a standalone guarantee.

\begin{table}[!t]
\centering
\caption{Independent MLLM audit of annotation quality on 5,000 randomly sampled items. Given the calibrated detection rates of the judge (Table~\ref{tab:app_judge}), these rates are a conservative floor on top of the human consensus review rather than a standalone guarantee.}
\label{tab:app_audit}
\vspace{-1.5mm}
\setlength{\tabcolsep}{3pt}
\begin{tabular*}{\columnwidth}{@{\extracolsep{\fill}}lccc@{}}
\toprule
 & \textbf{Test} & \textbf{Train} & \textbf{All} \\
\midrule
Answer correct & 100.0\% & 100.0\% & 100.0\% \\
Evidence boxes correct & 100.0\% & 100.0\% & 100.0\% \\
Both correct & 100.0\% & 100.0\% & 100.0\% \\
Count matches box cardinality & 100.0\% & 100.0\% & 100.0\% \\
\bottomrule
\end{tabular*}

\end{table}

\textit{Example items:} Table~\ref{tab:app_examples} shows representative items from the test split. The evidence set may contain several boxes, a single box, or none, and the answer is expected to agree with the box set whenever it refers to identifiable objects.

\begin{table}[!t]
\centering
\caption{Representative items from the RoadSceneVQA-G test split. Boxes: number of evidence boxes bound to the answer. An empty set (0) marks a no-ground question.}
\label{tab:app_examples}
\vspace{-1.5mm}
\setlength{\tabcolsep}{3pt}
\begin{tabular*}{\columnwidth}{@{\extracolsep{\fill}}>{\raggedright\arraybackslash}p{1.6cm}>{\raggedright\arraybackslash}p{3.45cm}lc@{}}
\toprule
\textbf{Category} & \textbf{Question} & \textbf{Answer} & \textbf{Boxes} \\
\midrule
Existence & Is there any pedestrian crossing the street? & Yes & 7 \\
Enumeration & How many cars are waiting at the intersection? & 3 & 3 \\
Attribute & What is the color of the car closest to the camera? & Black & 1 \\
Scene inference & What is the weather condition in the image? & Overcast & 0 \\
\bottomrule
\end{tabular*}
\end{table}

\textit{License and release:} The imagery and object labels derive from the Rope3D dataset~\cite{ye2022rope3d}, and the question coverage derives from RoadSceneVQA~\cite{guan2026roadscenevqa}. Both are used under their original licenses. We will release the QA annotations, the evidence bindings, and the evaluation code upon publication.

\section{Complete Baseline Results}
\label{app:full_baselines}

Table~\ref{tab:roadscene_mllm_full} reports the complete zero-shot and LoRA-SFT MLLM sweep on RoadSceneVQA-G, including text-generation metrics, grounding metrics, and inference latency. Table~\ref{tab:roadscene_grounded_full} reports the grounded MLLM baselines fine-tuned on the same training split. Table~\ref{tab:roadscene_vg_specialists} reports the VG specialists as box predictors on QA-derived expressions.

\begin{table*}[!t]
\centering
\caption{Detailed zero-shot and LoRA-SFT MLLM baselines on RoadSceneVQA-G. Text metrics improve after standard LoRA SFT, but grounding recall and F1 often collapse because the models are not required to read the answer out of the box set. GDA corr.: raw number of correct grounding decisions. CntAcc is reported together with its support CntSup. Time: inference latency per sample.}
\label{tab:roadscene_mllm_full}
\vspace{-1.5mm}
\setlength{\tabcolsep}{3.2pt}
\resizebox{\textwidth}{!}{%
\begin{tabular}{llrrrrrrrrrrr}
\toprule
\textbf{Model} & \textbf{Setting} & \textbf{CIDEr} & \textbf{BERT-F1} & \textbf{Prec} & \textbf{Rec} & \textbf{F1} & \textbf{AP} & \textbf{GDA} & \textbf{GDA corr.} & \textbf{CntAcc} & \textbf{CntSup} & \textbf{Time} \\
\midrule
InternVL3.5-8B-HF~\cite{wang2025internvl3} & zero-shot & 76.5 & 89.4 & 40.4 & 33.9 & 36.9 & 13.0 & 60.0 & 2805/4677 & 60.0 & 25 & -- \\
InternVL3.5-8B-HF~\cite{wang2025internvl3} & LoRA SFT & 123.6 & 91.8 & 56.5 & 8.4 & 14.6 & 9.1 & 25.5 & 1193/4677 & 58.7 & 579 & 1.36\,s \\
Qwen3.5-27B~\cite{qwen2026qwen35} & zero-shot, thinking & 37.0 & 86.1 & 37.1 & \textbf{72.9} & 49.1 & 26.1 & 79.5 & 3719/4677 & \textbf{82.3} & 277 & -- \\
Qwen3.5-27B~\cite{qwen2026qwen35} & zero-shot, non-thinking & 55.0 & 88.8 & \textbf{61.8} & 68.1 & \textbf{64.8} & \textbf{34.5} & 78.7 & 3682/4677 & 80.1 & 181 & -- \\
MiniCPM-o 4.5~\cite{cui2026minicpmo} & zero-shot, thinking & 45.7 & 85.7 & 9.0 & 2.5 & 3.9 & 0.4 & 45.8 & 2144/4677 & 38.5 & 143 & -- \\
MiniCPM-o 4.5~\cite{cui2026minicpmo} & zero-shot, non-thinking & 70.9 & 88.2 & 6.3 & 3.4 & 4.4 & 0.3 & 63.4 & 2964/4677 & 32.0 & 1265 & -- \\
MiniCPM-o 4.5~\cite{cui2026minicpmo} & LoRA SFT & \textbf{124.7} & 90.9 & 2.7 & 1.8 & 2.2 & 0.5 & 36.0 & 1683/4677 & 30.2 & 481 & 2.55\,s \\
MiniCPM-V 4.6~\cite{yao2024minicpm} & zero-shot, non-thinking & 54.6 & 87.5 & 36.0 & 0.1 & 0.2 & 6.4 & 16.3 & 760/4677 & 0.0 & 1 & -- \\
MiniCPM-V 4.6~\cite{yao2024minicpm} & LoRA SFT & 113.2 & \textbf{91.8} & 16.1 & 12.2 & 13.8 & 2.4 & \textbf{88.2} & \textbf{4125/4677} & 61.7 & 790 & 1.69\,s \\
LLaVA-OneVision-2-8B~\cite{an2026llavaonevision2} & zero-shot & 54.5 & 90.5 & 23.2 & 5.4 & 8.7 & 1.1 & 32.6 & 1526/4677 & 20.7 & 1192 & 2.05\,s \\
LLaVA-OneVision-2-8B~\cite{an2026llavaonevision2} & LoRA SFT & 92.5 & 91.7 & 5.8 & 0.6 & 1.1 & 0.6 & 25.0 & 1169/4677 & 28.6 & 1119 & 1.66\,s \\
\bottomrule
\end{tabular}}

\end{table*}

\begin{table*}[!t]
\centering
\caption{Grounded MLLM baselines on RoadSceneVQA-G. All models are fine-tuned on the same training split with their native box syntax. GCA is reported in its strict (GCA-s) and lenient (GCA-l) variants.}
\label{tab:roadscene_grounded_full}
\vspace{-1.5mm}
\setlength{\tabcolsep}{3.2pt}
\resizebox{\textwidth}{!}{%
\begin{tabular}{lrrrrrrrrrrrrrr}
\toprule
\textbf{Model} & \textbf{BLEU-4} & \textbf{CIDEr} & \textbf{ROUGE-L} & \textbf{METEOR} & \textbf{BERT-F1} & \textbf{Prec} & \textbf{Rec} & \textbf{F1} & \textbf{AP} & \textbf{GDA} & \textbf{SPC} & \textbf{CntAcc} & \textbf{GCA-s} & \textbf{GCA-l} \\
\midrule
Ferret-7B~\cite{you2024ferret} & 9.1 & 97.4 & 33.5 & 17.6 & 92.8 & 45.1 & 9.5 & 15.7 & 7.8 & 36.2 & 23.3 & 31.6 & 14.3 & 19.6 \\
Shikra-7B~\cite{chen2023shikra} & 10.4 & 110.3 & 34.6 & 20.0 & 92.4 & 56.8 & 10.9 & 18.3 & 14.3 & 42.3 & 28.1 & 37.8 & 15.1 & 22.5 \\
Groma-7B~\cite{ma2024groma} & 12.6 & 131.9 & 35.9 & 25.1 & 92.6 & 66.9 & 57.4 & 61.8 & 39.5 & 95.4 & 97.8 & 71.6 & 25.3 & 46.1 \\
\midrule
EtA-SFT (ours) & 13.9 & 146.7 & 39.7 & 26.6 & 92.6 & \textbf{82.6} & 65.4 & 73.0 & 54.4 & 88.0 & 93.7 & 65.9 & 29.6 & 51.2 \\
ECPO (ours) & \textbf{14.0} & \textbf{147.0} & \textbf{40.0} & \textbf{26.8} & 92.7 & 81.0 & \textbf{70.9} & \textbf{75.6} & \textbf{60.8} & \textbf{95.6} & \textbf{97.9} & \textbf{72.7} & \textbf{32.3} & \textbf{55.5} \\
\bottomrule
\end{tabular}}

\end{table*}

\begin{table*}[!t]
\centering
\caption{Visual grounding specialists on RoadSceneVQA-G. These methods are evaluated only as box predictors with the same QA-derived referring expressions as input. They produce no VQA answer, so answer-dependent metrics are not applicable.}
\label{tab:roadscene_vg_specialists}
\vspace{-1.5mm}
\setlength{\tabcolsep}{4pt}
\resizebox{\textwidth}{!}{%
\begin{tabular}{lllrrrrrrrr}
\toprule
\textbf{Model} & \textbf{Runtime} & \textbf{Checkpoint} & \textbf{Pred. boxes} & \textbf{Prec} & \textbf{Rec} & \textbf{F1} & \textbf{AP$_{50:95}$} & \textbf{GDA} & \textbf{GDA corr.} & \textbf{Time} \\
\midrule
GroundingDINO-T~\cite{liu2024grounding} & official & Swin-T OGC & 3,680 & 31.0 & 14.7 & 19.9 & 7.2 & 29.9 & 1398/4677 & 0.194\,s \\
GroundingDINO-T~\cite{liu2024grounding} & MMDetection, zero-shot & Swin-T OGC & 10,566 & 42.1 & 57.1 & 48.5 & 26.5 & 79.5 & 3716/4677 & 0.190\,s \\
GroundingDINO-T~\cite{liu2024grounding} & MMDetection, fine-tuned & epoch 12 & 8,965 & \textbf{67.9} & \textbf{78.2} & \textbf{72.7} & \textbf{62.7} & \textbf{86.3} & \textbf{4038/4677} & 0.177\,s \\
YOLO-World-v2-L 640~\cite{cheng2024yolo} & zero-shot & Obj365+GoldG & 12,797 & 0.1 & 0.2 & 0.2 & 0.0 & 85.3 & 3988/4677 & 0.037\,s \\
YOLO-World-v2-L 640~\cite{cheng2024yolo} & fine-tuned & epoch 24 & 5,730 & 23.3 & 17.2 & 19.8 & 3.0 & 51.4 & 2405/4677 & 0.037\,s \\
DINO-X~\cite{ren2024dino} & official API & -- & 16,180 & 29.7 & 61.8 & 40.2 & 27.9 & 76.5 & 3577/4677 & 1.934\,s \\
\bottomrule
\end{tabular}}

\end{table*}

\section{Additional Controls and Analyses}
\label{app:controls}

\subsection{Consensus-Set Statistics}
\label{app:consensus}

Table~\ref{tab:consensus} reports the full rollout-correlation and consensus-set statistics analyzed in Section~\ref{M-sec:ablation} of the main paper.

\begin{table}[!t]
\centering
\caption{Rollout correlation and consensus-set quality, measured with $G=6$ rollouts on 2,000 test prompts. Support $s$ counts the rollouts that recover each ground-truth object.}
\label{tab:consensus}
\vspace{-3mm}
\setlength{\tabcolsep}{3pt}
\begin{tabular*}{\columnwidth}{@{\extracolsep{\fill}}lc@{}}
\toprule
\textbf{Quantity} & \textbf{Value} \\
\midrule
Object share with $s=0$ / $1$ / $\ge k$ (\%) & 22.1 / 14.9 / 63.0 \\
Marginal per-rollout detection rate $\hat{p}$ & 0.416 \\
Union recall $R_{\mathcal{U}}$ & 77.9\% \\
Independence prediction $1-(1-\hat{p})^{G}$ & 96.0\% \\
Consensus P / R / F1 (overall) & 75.8 / 62.3 / 68.4 \\
Consensus P / R / F1 (7+ bucket) & 79.8 / 32.1 / 45.8 \\
Full-support clusters without GT match & 9.0\% \\
Mean support ratio, area Q1 to Q4 & 0.17 / 0.34 / 0.54 / 0.60 \\
\bottomrule
\end{tabular*}

\end{table}

\subsection{Distribution-Matched Re-Weighting of the Test Split}
\label{app:controls_reweight}

The training and test frames come from disjoint camera collections inherited from the source datasets, and the capability mix differs between the two splits (Section~\ref{M-sec:statistics}). To verify that the main gains are not an artifact of this mismatch, we re-weight every test sample by $w_c=p_{\mathrm{train}}(c)/p_{\mathrm{test}}(c)$ over the nine capability categories and recompute the metrics (Table~\ref{tab:app_reweight}). Under the matched distribution, ECPO remains the best of the three models on all four metrics, and its GCA-l margin over standard SFT is preserved (9.2 points raw and 9.3 points weighted), so the main conclusions do not depend on the split composition. The re-weighting mainly raises the SPC of standard SFT, because the training distribution contains fewer count-stating enumeration questions.

\begin{table}[!t]
\centering
\caption{Importance-weighted test metrics. Each test sample is weighted by the ratio of its capability-category share between the training and test splits (raw: unweighted, wtd: weighted). Box-level F1 is computed over all matched pairs and is unaffected by per-sample weights, so it is omitted.}
\label{tab:app_reweight}
\vspace{-1.5mm}
\setlength{\tabcolsep}{3pt}
\begin{tabular*}{\columnwidth}{@{\extracolsep{\fill}}lcccccccc@{}}
\toprule
 & \multicolumn{2}{c}{\textbf{GDA}} & \multicolumn{2}{c}{\textbf{SPC}} & \multicolumn{2}{c}{\textbf{CntAcc}} & \multicolumn{2}{c}{\textbf{GCA-l}} \\
\cmidrule(lr){2-3}\cmidrule(lr){4-5}\cmidrule(lr){6-7}\cmidrule(lr){8-9}
\textbf{Method} & raw & wtd & raw & wtd & raw & wtd & raw & wtd \\
\midrule
Standard SFT & 88.8 & 84.2 & 26.6 & 58.0 & 52.1 & 60.6 & 46.3 & 59.6 \\
EtA-SFT & 88.0 & 84.2 & 93.7 & 95.2 & 65.9 & 67.5 & 51.2 & 62.0 \\
ECPO & 95.6 & 92.9 & 97.9 & 97.2 & 72.7 & 78.4 & 55.5 & 68.9 \\
\bottomrule
\end{tabular*}

\end{table}

\subsection{Per-Category Metric Breakdown}
\label{app:per_category}

Using the capability labels in Table~\ref{M-tab:qtype}, we break down the three main model variants by question category (Table~\ref{tab:app_per_category}). The gains concentrate where the consistency requirement binds. On enumeration questions, ECPO raises SPC from 15.9 to 99.0 and F1 from 33.1 to 71.3 relative to standard SFT. On existence questions, it more than doubles CntAcc and raises GDA from 48.1 to 75.0. Categories dominated by single-object or global questions are close to saturation for all models, so the remaining headroom lies in the categories that require complete enumeration.

\begin{table*}[!t]
\centering
\caption{Per-category metrics on the test split (Std: standard SFT, EtA: EtA-SFT, ECPO: EtA-SFT + ECPO). Scene inference has no ground-truth boxes, so its grounding and consistency metrics are undefined. SPC on attribute recognition ($^{\dagger}$) has a small support because such answers rarely state a count.}
\label{tab:app_per_category}
\vspace{-1.5mm}
\setlength{\tabcolsep}{2.6pt}
\resizebox{\textwidth}{!}{%
\begin{tabular}{llr ccc ccc ccc ccc ccc}
\toprule
 & & & \multicolumn{3}{c}{\textbf{F1}} & \multicolumn{3}{c}{\textbf{GDA}} & \multicolumn{3}{c}{\textbf{SPC}} & \multicolumn{3}{c}{\textbf{CntAcc}} & \multicolumn{3}{c}{\textbf{GCA-l}} \\
\cmidrule(lr){4-6}\cmidrule(lr){7-9}\cmidrule(lr){10-12}\cmidrule(lr){13-15}\cmidrule(lr){16-18}
\textbf{Tier} & \textbf{Category} & \textbf{$n$} & Std & EtA & ECPO & Std & EtA & ECPO & Std & EtA & ECPO & Std & EtA & ECPO & Std & EtA & ECPO \\
\midrule
\multirow{5}{*}{Perception}
& Existence               & 509  & 32.2 & 50.4 & \textbf{64.8} & 48.1 & 52.7 & \textbf{75.0} & 76.8 & 95.2 & \textbf{96.0} & 22.4 & 33.4 & \textbf{57.0} & 35.8 & 43.4 & \textbf{62.9} \\
& Enumeration             & 1344 & 33.1 & 70.2 & \textbf{71.3} & 90.8 & 91.5 & \textbf{98.0} & 15.9 & 93.5 & \textbf{99.0} & 7.7 & 43.9 & \textbf{46.1} & 26.5 & 42.2 & \textbf{44.0} \\
& Attribute recognition   & 1440 & 87.8 & 87.1 & \textbf{90.2} & 96.9 & 94.4 & \textbf{99.0} & 50.0$^{\dagger}$ & \textbf{83.3}$^{\dagger}$ & 40.0$^{\dagger}$ & 94.2 & 91.6 & \textbf{98.0} & 78.5 & 77.4 & \textbf{81.1} \\
& Object identification   & 476  & 54.9 & 68.2 & \textbf{70.1} & \textbf{98.1} & 88.0 & 96.2 & 5.3 & \textbf{89.3} & 82.9 & 59.4 & 65.5 & \textbf{70.8} & 13.7 & 14.7 & \textbf{15.8} \\
& Scene inference         & 182  & -- & -- & -- & 100.0 & 100.0 & 100.0 & -- & -- & -- & -- & -- & -- & 82.4 & 80.8 & \textbf{81.9} \\
\midrule
\multirow{3}{*}{Reasoning}
& Spatial reasoning       & 487  & 88.8 & 90.4 & \textbf{94.4} & 87.9 & 90.3 & \textbf{97.7} & -- & -- & -- & 86.0 & 89.7 & \textbf{97.4} & 38.6 & 37.0 & \textbf{40.7} \\
& Behavioral reasoning    & 157  & 85.4 & 86.0 & \textbf{88.8} & 89.8 & 91.1 & \textbf{95.5} & 57.1 & 71.4 & \textbf{71.4} & 85.2 & 86.5 & \textbf{91.0} & 29.9 & 32.5 & \textbf{34.4} \\
& Normative reasoning     & 80   & 91.4 & 89.9 & \textbf{92.8} & 93.8 & 91.2 & \textbf{96.2} & 100.0 & 100.0 & 100.0 & 97.3 & 94.7 & \textbf{100.0} & \textbf{56.2} & 51.2 & 52.5 \\
\bottomrule
\end{tabular}}

\end{table*}

\subsection{Calibration of the Annotation Judge}
\label{app:controls_judge}

The independent audit in Appendix~\ref{app:dataset} relies on an MLLM judge, so we calibrate the judge itself. We draw 200 accepted test items, inject known errors of three types (wrong evidence boxes sampled from other scenes, a corrupted numeric answer, and an extra spurious box), and ask the same judge to audit the perturbed items. Table~\ref{tab:app_judge} reports the detection rate per perturbation type for two judge scales. The larger judge detects extra boxes reliably and catches about half of the swapped evidence sets. The audit numbers in Appendix~\ref{app:dataset} are therefore best read as a conservative floor combined with the human-enforced consistency contract described in Section~\ref{M-sec:construction}, rather than as a standalone guarantee. The annotation contract itself was enforced by human reviewers at review time and is independent of the judge.

\begin{table}[!t]
\centering
\caption{Calibration of the MLLM judge: detection rate of injected errors on 200 perturbed items for two judge scales of Qwen3-VL~\cite{bai2025qwen3}.}
\label{tab:app_judge}
\vspace{-1.5mm}
\setlength{\tabcolsep}{3pt}
\begin{tabular*}{\columnwidth}{@{\extracolsep{\fill}}lcc@{}}
\toprule
\textbf{Perturbation} & \textbf{Qwen3-VL-Plus} & \textbf{Qwen3-VL-235B} \\
\midrule
Wrong evidence boxes & 26.7\% & 51.1\% \\
Wrong count in answer & 50.0\% & 47.4\% \\
Extra spurious box & 83.3\% & 83.1\% \\
\bottomrule
\end{tabular*}
\end{table}

\section{Additional Qualitative Results on CARPK}
\label{app:qualitative}

Qualitative results on gRefCOCO are shown in Fig.~\ref{M-fig:grefcoco_main} of the main paper. Fig.~\ref{fig:carpk_vis} compares predictions on dense parking-lot scenes from CARPK. EtA-SFT already localizes most vehicles but omits a small fraction in the most crowded rows, recovering 29 of 31 cars with an F1 of 0.90. Vanilla GRPO does not close this gap and enumerates even fewer instances (27 of 31), because its reward provides no signal for objects that all rollouts omit jointly. ECPO produces the tightest and most complete box set, recovering 29 of 31 cars with an F1 of 0.97, which is consistent with the group-union mechanism that pools complementary detections across rollouts.

\begin{figure*}[!t]
\centering
\includegraphics[width=\textwidth]{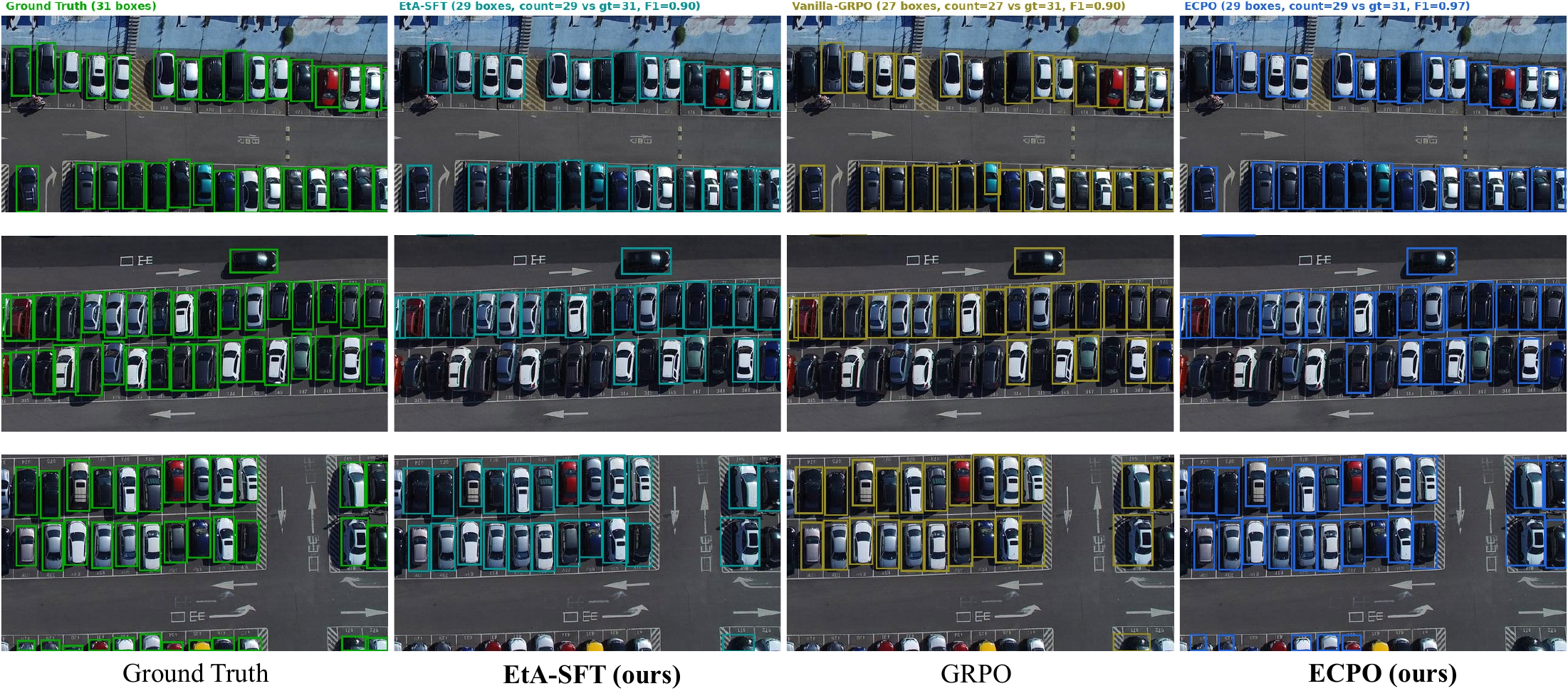}
\vspace{-3.5mm}
\caption{Qualitative comparison on CARPK. Boxes predicted by EtA-SFT, vanilla GRPO, and ECPO. ECPO recovers the most complete box set.}
\label{fig:carpk_vis}
\end{figure*}

\bibliographystyle{IEEEtran}
\bibliography{ref}

@inproceedings{zhou2025tumtraffic,
  title={Tumtraf videoqa: Dataset and benchmark for unified spatio-temporal video understanding in traffic scenes},
  author={Zhou, Xingcheng and Larintzakis, Konstantinos and Guo, Hao and Zimmer, Walter and Liu, Mingyu and Cao, Hu and Zhang, Jiajie and Lakshminarasimhan, Venkatnarayanan and Strand, Leah and Knoll, Alois},
  booktitle={Forty-second International Conference on Machine Learning},
  year={2025}
}

@inproceedings{sima2024drivelm,
  title={{DriveLM}: Driving with graph visual question answering},
  author={Sima, Chonghao and Renz, Katrin and Chitta, Kashyap and Chen, Li and Zhang, Hanxue and Xie, Chengen and Bei{\ss}wenger, Jens and Luo, Ping and Geiger, Andreas and Li, Hongyang},
  booktitle={European Conference on Computer Vision},
  pages={256--274},
  year={2024},
  organization={Springer}
}

@inproceedings{zhang2025mpdrive,
  title={{MPDrive}: Improving spatial understanding with marker-based prompt learning for autonomous driving},
  author={Zhang, Zhiyuan and Li, Xiaofan and Xu, Zhihao and Peng, Wenjie and Zhou, Zijian and Shi, Miaojing and Huang, Shuangping},
  booktitle={Proceedings of the Computer Vision and Pattern Recognition Conference},
  pages={12089--12099},
  year={2025}
}

@inproceedings{wu2025language,
  title={Language prompt for autonomous driving},
  author={Wu, Dongming and Han, Wencheng and Liu, Yingfei and Wang, Tiancai and Xu, Cheng-zhong and Zhang, Xiangyu and Shen, Jianbing},
  booktitle={Proceedings of the AAAI Conference on Artificial Intelligence},
  volume={39},
  number={8},
  pages={8359--8367},
  year={2025}
}

@inproceedings{xie2025vlms,
  title={Are {VLMs} Ready for Autonomous Driving? An Empirical Study from the Reliability, Data and Metric Perspectives},
  author={Xie, Shaoyuan and Kong, Lingdong and Dong, Yuhao and Sima, Chonghao and Zhang, Wenwei and Chen, Qi Alfred and Liu, Ziwei and Pan, Liang},
  booktitle={Proceedings of the IEEE/CVF International Conference on Computer Vision},
  pages={6585--6597},
  year={2025}
}

@inproceedings{ye2022rope3d,
  title={{Rope3D}: The roadside perception dataset for autonomous driving and monocular {3D} object detection task},
  author={Ye, Xiaoqing and Shu, Mao and Li, Hanyu and Shi, Yifeng and Li, Yingying and Wang, Guangjie and Tan, Xiao and Ding, Errui},
  booktitle={Proceedings of the IEEE/CVF Conference on Computer Vision and Pattern Recognition},
  pages={21341--21350},
  year={2022}
}

@inproceedings{yu2022dair,
  title={{DAIR-V2X}: A large-scale dataset for vehicle-infrastructure cooperative {3D} object detection},
  author={Yu, Haibao and Luo, Yizhen and Shu, Mao and Huo, Yiyi and Yang, Zebang and Shi, Yifeng and Guo, Zhenglong and Li, Hanyu and Hu, Xing and Yuan, Jirui and others},
  booktitle={Proceedings of the IEEE/CVF Conference on Computer Vision and Pattern Recognition},
  pages={21361--21370},
  year={2022}
}

@article{bai2025qwen2,
  title={{Qwen2.5-VL} technical report},
  author={Bai, Shuai and Chen, Keqin and Liu, Xuejing and Wang, Jialin and Ge, Wenbin and Song, Sibo and Dang, Kai and Wang, Peng and Wang, Shijie and Tang, Jun and others},
  journal={arXiv preprint arXiv:2502.13923},
  year={2025}
}

@article{yao2024minicpm,
  title={{MiniCPM-V}: A {GPT-4V} level {MLLM} on your phone},
  author={Yao, Yuan and Yu, Tianyu and Zhang, Ao and Wang, Chongyi and Cui, Junbo and Zhu, Hongji and Cai, Tianchi and Li, Haoyu and Zhao, Weilin and He, Zhihui and others},
  journal={arXiv preprint arXiv:2408.01800},
  year={2024}
}

@inproceedings{papineni2002bleu,
  title={{BLEU}: A method for automatic evaluation of machine translation},
  author={Papineni, Kishore and Roukos, Salim and Ward, Todd and Zhu, Wei-Jing},
  booktitle={Proceedings of the 40th annual meeting of the Association for Computational Linguistics},
  pages={311--318},
  year={2002}
}

@inproceedings{lin2004rouge,
  title={{ROUGE}: A package for automatic evaluation of summaries},
  author={Lin, Chin-Yew},
  booktitle={Text summarization branches out},
  pages={74--81},
  year={2004}
}

@inproceedings{vedantam2015cider,
  title={{CIDEr}: Consensus-based image description evaluation},
  author={Vedantam, Ramakrishna and Lawrence Zitnick, C and Parikh, Devi},
  booktitle={Proceedings of the IEEE conference on computer vision and pattern recognition},
  pages={4566--4575},
  year={2015}
}

@inproceedings{meteor,
  author    = {Lavie, Alon and Agarwal, Abhaya},
  title     = {{METEOR}: An Automatic Metric for {MT} Evaluation with High Levels of Correlation with Human Judgments},
  booktitle = {Proceedings of the Second Workshop on Statistical Machine Translation},
  year      = {2007},
  pages     = {228--231},
  doi       = {10.3115/1626355.1626389}
}

@article{zhou2024vision,
  title={Vision language models in autonomous driving: A survey and outlook},
  author={Zhou, Xingcheng and Liu, Mingyu and Yurtsever, Ekim and Zagar, Bare Luka and Zimmer, Walter and Cao, Hu and Knoll, Alois C},
  journal={IEEE Transactions on Intelligent Vehicles},
  year={2024},
  publisher={IEEE}
}

@inproceedings{deruyttere2019talk2car,
  title={{Talk2Car}: Taking Control of Your {Self-Driving} Car},
  author={Deruyttere, Thierry and Vandenhende, Simon and Grujicic, Dusan and Van Gool, Luc and Moens, Marie Francine},
  booktitle={Proceedings of the 2019 Conference on Empirical Methods in Natural Language Processing and the 9th International Joint Conference on Natural Language Processing (EMNLP-IJCNLP)},
  pages={2088--2098},
  year={2019}
}

@inproceedings{guan2024talk2radar,
  title={{Talk2Radar}: Bridging natural language with {4D} {mmWave} radar for {3D} referring expression comprehension},
  author={Guan, Runwei and Zhang, Ruixiao and Ouyang, Ningwei and Liu, Jianan and Man, Ka Lok and Cai, Xiaohao and Xu, Ming and Smith, Jeremy and Lim, Eng Gee and Yue, Yutao and others},
  booktitle={Proceedings of the IEEE International Conference on Robotics and Automation (ICRA)},
  year={2025}
}

@inproceedings{guan2026roadscenevqa,
  title={{RoadSceneVQA}: Benchmarking visual question answering in roadside perception systems for intelligent transportation system},
  author={Guan, Runwei and Hu, Rongsheng and Chen, Shangshu and Xiao, Ningyuan and Xia, Xue and Liu, Jiayang and Chen, Beibei and Tang, Ziren and Ouyang, Ningwei and Liang, Shaofeng and others},
  booktitle={Proceedings of the AAAI Conference on Artificial Intelligence},
  volume={40},
  number={6},
  pages={4366--4375},
  year={2026}
}

@article{chen2023shikra,
  title={{Shikra}: Unleashing multimodal {LLM}'s referential dialogue magic},
  author={Chen, Keqin and Zhang, Zhao and Zeng, Weili and Zhang, Richong and Zhu, Feng and Zhao, Rui},
  journal={arXiv preprint arXiv:2306.15195},
  year={2023}
}

@inproceedings{you2024ferret,
  title={{Ferret}: Refer and ground anything anywhere at any granularity},
  author={You, Haoxuan and Zhang, Haotian and Gan, Zhe and Du, Xianzhi and Zhang, Bowen and Wang, Zirui and Cao, Liangliang and Chang, Shih-Fu and Yang, Yinfei},
  booktitle={International Conference on Learning Representations},
  pages={57153--57180},
  year={2024}
}

@article{bai2023qwenvlversatilevisionlanguagemodel,
      journal={arXiv preprint arXiv:2308.12966},
      title={{Qwen-VL}: A Versatile {Vision-Language} Model for Understanding, Localization, Text Reading, and Beyond}, 
      author={Jinze Bai and Shuai Bai and Shusheng Yang and Shijie Wang and Sinan Tan and Peng Wang and Junyang Lin and Chang Zhou and Jingren Zhou},
      year={2023},
}

@article{bai2025qwen3,
  title={{Qwen3-VL} technical report},
  author={Bai, Shuai and Cai, Yuxuan and Chen, Ruizhe and Chen, Keqin and Chen, Xionghui and Cheng, Zesen and Deng, Lianghao and Ding, Wei and Gao, Chang and Ge, Chunjiang and others},
  journal={arXiv preprint arXiv:2511.21631},
  year={2025}
}

@inproceedings{chen2025automated,
  title={Automated evaluation of large vision-language models on self-driving corner cases},
  author={Chen, Kai and Li, Yanze and Zhang, Wenhua and Liu, Yanxin and Li, Pengxiang and Gao, Ruiyuan and Hong, Lanqing and Tian, Meng and Zhao, Xinhai and Li, Zhenguo and others},
  booktitle={2025 IEEE/CVF Winter Conference on Applications of Computer Vision (WACV)},
  pages={7817--7826},
  year={2025},
  organization={IEEE}
}

@article{ding2026grex,
  title={{GREx}: Generalized referring expression segmentation, comprehension, and generation},
  author={Ding, Henghui and Liu, Chang and He, Shuting and Jiang, Xudong and Jiang, Yu-Gang},
  journal={International Journal of Computer Vision},
  volume={134},
  number={2},
  pages={79},
  year={2026},
  publisher={Springer}
}

@inproceedings{peng2024grounding,
  title={Grounding multimodal large language models to the world},
  author={Peng, Zhiliang and Wang, Wenhui and Dong, Li and Hao, Yaru and Huang, Shaohan and Ma, Shuming and Ye, Qixiang and Wei, Furu},
  booktitle={International Conference on Learning Representations},
  pages={51575--51598},
  year={2024}
}

@inproceedings{wu2025routing,
  title={Routing experts: Learning to route dynamic experts in existing multi-modal large language models},
  author={Wu, Qiong and Ke, Zhaoxi and Zhou, Yiyi and Sun, Xiaoshuai and Ji, Rongrong},
  booktitle={The Thirteenth International Conference on Learning Representations},
  year={2025}
}

@article{fan2026grit,
  title={{GRIT}: Teaching {MLLMs} to think with images},
  author={Fan, Yue and He, Xuehai and Yang, Diji and Zheng, Kaizhi and Kuo, Ching-Chen and Zheng, Yuting and Guan, Xinze and Wang, Xin},
  journal={Advances in Neural Information Processing Systems},
  volume={38},
  pages={116522--116543},
  year={2026}
}

@article{yin2024woodpecker,
  title={{Woodpecker}: Hallucination correction for multimodal large language models},
  author={Yin, Shukang and Fu, Chaoyou and Zhao, Sirui and Xu, Tong and Wang, Hao and Sui, Dianbo and Shen, Yunhang and Li, Ke and Sun, Xing and Chen, Enhong},
  journal={Science China Information Sciences},
  volume={67},
  number={12},
  pages={220105},
  year={2024},
  publisher={Springer}
}

@inproceedings{wu2024v,
  title={{V}$^{*}$: Guided visual search as a core mechanism in multimodal {LLMs}},
  author={Wu, Penghao and Xie, Saining},
  booktitle={Proceedings of the IEEE/CVF Conference on Computer Vision and Pattern Recognition},
  pages={13084--13094},
  year={2024}
}

@inproceedings{ranjan2021learning,
  title={Learning to count everything},
  author={Ranjan, Viresh and Sharma, Udbhav and Nguyen, Thu and Hoai, Minh},
  booktitle={Proceedings of the IEEE/CVF conference on computer vision and pattern recognition},
  pages={3394--3403},
  year={2021}
}

@inproceedings{jiang2023clip,
  title={{CLIP-Count}: Towards text-guided zero-shot object counting},
  author={Jiang, Ruixiang and Liu, Lingbo and Chen, Changwen},
  booktitle={Proceedings of the 31st ACM International Conference on Multimedia},
  pages={4535--4545},
  year={2023}
}

@article{shao2024deepseekmath,
  title={{DeepSeekMath}: Pushing the limits of mathematical reasoning in open language models},
  author={Shao, Zhihong and Wang, Peiyi and Zhu, Qihao and Xu, Runxin and Song, Junxiao and Bi, Xiao and Zhang, Haowei and Zhang, Mingchuan and Li, YK and others},
  journal={arXiv preprint arXiv:2402.03300},
  year={2024}
}

@article{guo2025deepseek,
  title={{DeepSeek-R1} incentivizes reasoning in {LLMs} through reinforcement learning},
  author={Guo, Daya and Yang, Dejian and Zhang, Haowei and Song, Junxiao and Wang, Peiyi and Zhu, Qihao and Xu, Runxin and Zhang, Ruoyu and Ma, Shirong and Bi, Xiao and others},
  journal={Nature},
  volume={645},
  number={8081},
  pages={633--638},
  year={2025},
  publisher={Nature Publishing Group UK London}
}

@article{bai2025univg,
  title={{UniVG-R1}: Reasoning guided universal visual grounding with reinforcement learning},
  author={Bai, Sule and Li, Mingxing and Liu, Yong and Tang, Jing and Zhang, Haoji and Sun, Lei and Chu, Xiangxiang and Tang, Yansong},
  journal={arXiv preprint arXiv:2505.14231},
  year={2025}
}

@article{shen2025vlm,
  title={{VLM-R1}: A stable and generalizable {R1}-style large vision-language model},
  author={Shen, Haozhan and Liu, Peng and Li, Jingcheng and Fang, Chunxin and Ma, Yibo and Liao, Jiajia and Shen, Qiaoli and Zhang, Zilun and Zhao, Kangjia and Zhang, Qianqian and others},
  journal={arXiv preprint arXiv:2504.07615},
  year={2025}
}

@article{batra2025spatialthinker,
  title={{SpatialThinker}: Reinforcing {3D} Reasoning in Multimodal {LLMs} via Spatial Rewards},
  author={Batra, Hunar and Tu, Haoqin and Chen, Hardy and Lin, Yuanze and Xie, Cihang and Clark, Ronald},
  journal={arXiv preprint arXiv:2511.07403},
  year={2025}
}

@article{cao2025ground,
  title={{Ground-R1}: Incentivizing grounded visual reasoning via reinforcement learning},
  author={Cao, Meng and Zhao, Haoze and Zhang, Can and Chang, Xiaojun and Reid, Ian and Liang, Xiaodan},
  journal={arXiv preprint arXiv:2505.20272},
  year={2025}
}

@article{wan2026srpo,
  title={{SRPO}: Enhancing multimodal {LLM} reasoning via reflection-aware reinforcement learning},
  author={Wan, Zhongwei and Dou, Zhihao and Liu, Che and Zhang, Yu and Cui, Dongfei and Zhao, Qinjian and Shen, Hui and Xiong, Jing and Xin, Yi and Jiang, Yifan and others},
  journal={Advances in Neural Information Processing Systems},
  volume={38},
  pages={153676--153713},
  year={2026}
}

@inproceedings{wangself,
  title={{Self-Consistency} Improves Chain of Thought Reasoning in Language Models},
  author={Wang, Xuezhi and Wei, Jason and Schuurmans, Dale and Le, Quoc V and Chi, Ed H and Narang, Sharan and Chowdhery, Aakanksha and Zhou, Denny},
  booktitle={The Eleventh International Conference on Learning Representations},
  year={2023}
}

@inproceedings{chen2026grpo,
  title={{GRPO-CARE}: Consistency-aware reinforcement learning for multimodal reasoning},
  author={Chen, Yi and Ge, Yuying and Wang, Rui and Ge, Yixiao and Cheng, Junhao and Shan, Ying and Liu, Xihui},
  booktitle={Findings of the Association for Computational Linguistics: ACL 2026},
  pages={4305--4320},
  year={2026}
}

@article{zuo2026ttrl,
  title={{TTRL}: Test-time reinforcement learning},
  author={Zuo, Yuxin and Zhang, Kaiyan and Sheng, Li and Qu, Shang and Cui, Ganqu and Zhu, Xuekai and Li, Haozhan and Long, Xinwei and Hua, Ermo and Qi, Biqing and others},
  journal={Advances in Neural Information Processing Systems},
  volume={38},
  pages={131459--131483},
  year={2026}
}

@article{park2025dip,
  title={{DIP-R1}: Deep inspection and perception with {RL} looking through and understanding complex scenes},
  author={Park, Sungjune and Kim, Hyunjun and Kim, Junho and Kim, Seongho and Ro, Yong Man},
  journal={arXiv preprint arXiv:2505.23179},
  year={2025}
}

@inproceedings{qian2024nuscenes,
  title={{NuScenes-QA}: A multi-modal visual question answering benchmark for autonomous driving scenario},
  author={Qian, Tianwen and Chen, Jingjing and Zhuo, Linhai and Jiao, Yang and Jiang, Yu-Gang},
  booktitle={Proceedings of the AAAI Conference on Artificial Intelligence},
  volume={38},
  number={5},
  pages={4542--4550},
  year={2024}
}

@inproceedings{hudson2019gqa,
  title={{GQA}: A new dataset for real-world visual reasoning and compositional question answering},
  author={Hudson, Drew A and Manning, Christopher D},
  booktitle={Proceedings of the IEEE/CVF conference on computer vision and pattern recognition},
  pages={6700--6709},
  year={2019}
}

@inproceedings{yu2016modeling,
  title={Modeling context in referring expressions},
  author={Yu, Licheng and Poirson, Patrick and Yang, Shan and Berg, Alexander C and Berg, Tamara L},
  booktitle={European conference on computer vision},
  pages={69--85},
  year={2016},
  organization={Springer}
}

@inproceedings{lin2014microsoft,
  title={{Microsoft COCO}: Common objects in context},
  author={Lin, Tsung-Yi and Maire, Michael and Belongie, Serge and Hays, James and Perona, Pietro and Ramanan, Deva and Doll{\'a}r, Piotr and Zitnick, C Lawrence},
  booktitle={European conference on computer vision},
  pages={740--755},
  year={2014},
  organization={Springer}
}

@inproceedings{rohrbach2018object,
  title={Object hallucination in image captioning},
  author={Rohrbach, Anna and Hendricks, Lisa Anne and Burns, Kaylee and Darrell, Trevor and Saenko, Kate},
  booktitle={Proceedings of the 2018 conference on empirical methods in natural language processing},
  pages={4035--4045},
  year={2018}
}

@inproceedings{li2023evaluating,
  title={Evaluating object hallucination in large vision-language models},
  author={Li, Yifan and Du, Yifan and Zhou, Kun and Wang, Jinpeng and Zhao, Xin and Wen, Ji-Rong},
  booktitle={Proceedings of the 2023 conference on empirical methods in natural language processing},
  pages={292--305},
  year={2023}
}

@article{wang2025internvl3,
  title={{InternVL3.5}: Advancing open-source multimodal models in versatility, reasoning, and efficiency},
  author={Wang, Weiyun and Gao, Zhangwei and Gu, Lixin and Pu, Hengjun and Cui, Long and Wei, Xingguang and Liu, Zhaoyang and Jing, Linglin and Ye, Shenglong and Shao, Jie and others},
  journal={arXiv preprint arXiv:2508.18265},
  year={2025}
}

@article{ren2024dino,
  title={{DINO-X}: A unified vision model for open-world object detection and understanding},
  author={Ren, Tianhe and Chen, Yihao and Jiang, Qing and Zeng, Zhaoyang and Xiong, Yuda and Liu, Wenlong and Ma, Zhengyu and Shen, Junyi and Gao, Yuan and Jiang, Xiaoke and others},
  journal={arXiv preprint arXiv:2411.14347},
  year={2024}
}

@inproceedings{liu2024grounding,
  title={{Grounding DINO}: Marrying {DINO} with grounded pre-training for open-set object detection},
  author={Liu, Shilong and Zeng, Zhaoyang and Ren, Tianhe and Li, Feng and Zhang, Hao and Yang, Jie and Jiang, Qing and Li, Chunyuan and Yang, Jianwei and Su, Hang and others},
  booktitle={European conference on computer vision},
  pages={38--55},
  year={2024},
  organization={Springer}
}

@inproceedings{cheng2024yolo,
  title={{YOLO-World}: Real-time open-vocabulary object detection},
  author={Cheng, Tianheng and Song, Lin and Ge, Yixiao and Liu, Wenyu and Wang, Xinggang and Shan, Ying},
  booktitle={2024 IEEE/CVF Conference on Computer Vision and Pattern Recognition (CVPR)},
  pages={16901--16911},
  year={2024},
  organization={IEEE}
}

@inproceedings{ma2024groma,
  title={{Groma}: Localized visual tokenization for grounding multimodal large language models},
  author={Ma, Chuofan and Jiang, Yi and Wu, Jiannan and Yuan, Zehuan and Qi, Xiaojuan},
  booktitle={European Conference on Computer Vision},
  pages={417--435},
  year={2024},
  organization={Springer}
}

@article{yu2026perception,
  title={{Perception-R1}: Pioneering perception policy with reinforcement learning},
  author={Yu, En and Lin, Kangheng and Zhao, Liang and Wei, Yana and Peng, Yuang and Wei, Haoran and Sun, Jianjian and Han, Chunrui and Ge, Zheng and Zhang, Xiangyu and others},
  journal={Advances in Neural Information Processing Systems},
  volume={38},
  pages={94827--94853},
  year={2026}
}

@inproceedings{liu2026visionreasoner,
  title={{VisionReasoner}: Unified reasoning-integrated visual perception via reinforcement learning},
  author={Liu, Yuqi and Qu, Tianyuan and Zhong, Zhisheng and Peng, Bohao and Liu, Shu and Yu, Bei and Jia, Jiaya},
  booktitle={International Conference on Learning Representations},
  pages={94069--94086},
  year={2026}
}

@inproceedings{hsieh2017drone,
  title={Drone-based object counting by spatially regularized regional proposal network},
  author={Hsieh, Meng-Ru and Lin, Yen-Liang and Hsu, Winston H},
  booktitle={2017 IEEE International Conference on Computer Vision (ICCV)},
  pages={4165--4173},
  year={2017},
  organization={IEEE}
}

@article{shao2024visual,
  title={{Visual CoT}: Advancing multi-modal language models with a comprehensive dataset and benchmark for chain-of-thought reasoning},
  author={Shao, Hao and Qian, Shengju and Xiao, Han and Song, Guanglu and Zong, Zhuofan and Wang, Letian and Liu, Yu and Li, Hongsheng},
  journal={Advances in Neural Information Processing Systems},
  volume={37},
  pages={8612--8642},
  year={2024}
}

@inproceedings{cogcom2024,
 author = {Qi, Ji and Ding, Ming and Wang, Weihan and Bai, Yushi and Lv, Qingsong and Hong, Wenyi and Bin, Xu and Hou, Lei and Li, Juanzi and Dong, Yuxiao and Tang, Jie},
 booktitle = {International Conference on Learning Representations},
 pages = {11090--11110},
 title = {{CogCoM}: A Visual Language Model with {Chain-of-Manipulations} Reasoning},
 year = {2025}
}

@inproceedings{ding2021vision,
  author    = {Ding, Henghui and Liu, Chang and Wang, Suchen and Jiang, Xudong},
  title     = {Vision-Language Transformer and Query Generation for Referring Segmentation},
  booktitle = {Proceedings of the IEEE/CVF International Conference on Computer Vision},
  year      = {2021},
  pages     = {16301--16310},
  doi       = {10.1109/ICCV48922.2021.01601}
}

@inproceedings{kamath2021mdetr,
  author    = {Kamath, Aishwarya and Singh, Mannat and LeCun, Yann and Synnaeve, Gabriel and Misra, Ishan and Carion, Nicolas},
  title     = {{MDETR} - Modulated Detection for End-to-End Multi-Modal Understanding},
  booktitle = {Proceedings of the IEEE/CVF International Conference on Computer Vision},
  year      = {2021},
  pages     = {1760--1770},
  doi       = {10.1109/ICCV48922.2021.00180}
}

@inproceedings{yan2023universal,
  author    = {Yan, Bin and Jiang, Yi and Wu, Jiannan and Wang, Dong and Luo, Ping and Yuan, Zehuan and Lu, Huchuan},
  title     = {Universal Instance Perception as Object Discovery and Retrieval},
  booktitle = {Proceedings of the IEEE/CVF Conference on Computer Vision and Pattern Recognition},
  year      = {2023},
  pages     = {15325--15336},
  doi       = {10.1109/CVPR52729.2023.01471}
}

@inproceedings{liu2023gres,
  author    = {Liu, Chang and Ding, Henghui and Jiang, Xudong},
  title     = {{GRES}: Generalized Referring Expression Segmentation},
  booktitle = {Proceedings of the IEEE/CVF Conference on Computer Vision and Pattern Recognition},
  year      = {2023},
  pages     = {23592--23601},
  doi       = {10.1109/CVPR52729.2023.02259}
}

@inproceedings{fu2025llmdet,
  author    = {Fu, Shenghao and Yang, Qize and Mo, Qijie and Yan, Junkai and Wei, Xihan and Meng, Jingke and Xie, Xiaohua and Zheng, Wei-Shi},
  title     = {{LLMDet}: Learning Strong Open-Vocabulary Object Detectors under the Supervision of Large Language Models},
  booktitle = {Proceedings of the IEEE/CVF Conference on Computer Vision and Pattern Recognition},
  year      = {2025},
  pages     = {14987--14997},
  doi       = {10.1109/CVPR52734.2025.01396}
}

@article{zhu2025internvl3,
  author  = {Zhu, Jinguo and Wang, Weiyun and Chen, Zhe and Liu, Zhaoyang and Ye, Shenglong and Gu, Lixin and Tian, Hao and Duan, Yuchen and Su, Weijie and Shao, Jie and others},
  title   = {{InternVL3}: Exploring Advanced Training and Test-Time Recipes for Open-Source Multimodal Models},
  journal = {arXiv preprint arXiv:2504.10479},
  year    = {2025}
}

@misc{qwen2026qwen35,
  author       = {{Qwen Team}},
  title        = {{Qwen3.5}: Towards Native Multimodal Agents},
  howpublished = {\url{https://qwen.ai/blog?id=qwen3.5}},
  year         = {2026}
}

@article{an2026llavaonevision2,
  author  = {An, Xiang and Xie, Yin and Tang, Feilong and Yan, Yunyao and Tan, Huajie and Zhu, Didi and Chen, Changrui and Zhao, Xiuwei and Qin, Bin and Yang, Kaicheng and others},
  title   = {{LLaVA-OneVision-2}: Towards Next-Generation Perceptual Intelligence},
  journal = {arXiv preprint arXiv:2605.25979},
  year    = {2026}
}

@article{cui2026minicpmo,
  author  = {Cui, Junbo and Xu, Bokai and Wang, Chongyi and Yu, Tianyu and Sun, Weiyue and Xu, Yingjing and Wang, Tianran and He, Zhihui and Ma, Wenshuo and Cai, Tianchi and others},
  title   = {{MiniCPM-o} 4.5: Towards Real-Time Full-Duplex Omni-Modal Interaction},
  journal = {arXiv preprint arXiv:2604.27393},
  year    = {2026}
}

@misc{openai2025gpt5,
  author       = {{OpenAI}},
  title        = {{GPT-5} System Card},
  howpublished = {\url{https://openai.com/index/gpt-5-system-card/}},
  year         = {2025}
}

@article{dearruda2022counting,
  author  = {de Arruda, Mauro dos Santos and Osco, Lucas Prado and Acosta, Plabiany Rodrigo and Gon{\c{c}}alves, Diogo Nunes and Marcato Junior, Jos{\'e} and Ramos, Ana Paula Marques and Matsubara, Edson Takashi and Luo, Zhipeng and Li, Jonathan and de Andrade Silva, Jonathan and Gon{\c{c}}alves, Wesley Nunes},
  title   = {Counting and Locating High-Density Objects Using Convolutional Neural Network},
  journal = {Expert Systems with Applications},
  volume  = {195},
  pages   = {116555},
  year    = {2022},
  doi     = {10.1016/j.eswa.2022.116555}
}

@article{liu2026rar,
  title={Rar: Retrieving and ranking augmented mllms for visual recognition},
  author={Liu, Ziyu and Sun, Zeyi and Zang, Yuhang and Li, Wei and Zhang, Pan and Dong, Xiaoyi and Xiong, Yuanjun and Lin, Dahua and Wang, Jiaqi},
  journal={IEEE Transactions on Image Processing},
  year={2026},
  publisher={IEEE}
}

@article{fu2025boosting,
  title={Boosting Faithful Multi-Modal LLMs via Complementary Visual Grounding},
  author={Fu, Zheren and Mao, Zhendong and Zhang, Lei and Zhang, Yongdong},
  journal={IEEE Transactions on Image Processing},
  volume={34},
  pages={8641--8655},
  year={2025},
  publisher={IEEE}
}

@article{suo2022rethinking,
  title={Rethinking and improving feature pyramids for one-stage referring expression comprehension},
  author={Suo, Wei and Sun, Mengyang and Wang, Peng and Zhang, Yanning and Wu, Qi},
  journal={IEEE Transactions on Image Processing},
  volume={32},
  pages={854--864},
  year={2022},
  publisher={IEEE}
}

@article{xie2024knowledge,
  title={Knowledge-augmented visual question answering with natural language explanation},
  author={Xie, Jiayuan and Cai, Yi and Chen, Jiali and Xu, Ruohang and Wang, Jiexin and Li, Qing},
  journal={IEEE Transactions on Image Processing},
  volume={33},
  pages={2652--2664},
  year={2024},
  publisher={IEEE}
}

@article{du2023redesigning,
  title={Redesigning multi-scale neural network for crowd counting},
  author={Du, Zhipeng and Shi, Miaojing and Deng, Jiankang and Zafeiriou, Stefanos},
  journal={IEEE Transactions on Image Processing},
  volume={32},
  pages={3664--3678},
  year={2023},
  publisher={IEEE}
}

\end{document}